\documentclass[journal,12pt,draftclsnofoot,onecolumn]{IEEEtran}
\usepackage{mathrsfs}
\usepackage{amsmath}
\usepackage{amsthm}
\usepackage{tabularx}
\usepackage{booktabs}
\usepackage{amssymb}
\usepackage{graphicx}
\usepackage{times}
\usepackage{latexsym}
\usepackage{subfigure}
\usepackage{cite}
\usepackage{balance}
\usepackage{multirow}
\usepackage{algorithmic}
\usepackage{color}
\usepackage{float}
\usepackage{cases}
\usepackage{algorithm}
\usepackage{yhmath}
\usepackage{footnote}
\usepackage{bm}
\usepackage[justification=centering]{caption}
\newtheorem{prop}{Proposition}
\newtheorem{lemma}{Lemma}
\newtheorem{theorem}{Theorem}

\theoremstyle{remark}
\newtheorem{remark}{Remark}
\renewcommand{\algorithmicrequire}{ \textbf{Input:}}

\begin{document}
\title{Joint Optimization of Communications and Federated Learning Over the Air}
\author{ Xin Fan$^{1}$, Yue Wang$^2$,~\IEEEmembership{Member,~IEEE}, Yan Huo$^{1}$,~\IEEEmembership{Senior Member,~IEEE}, and Zhi Tian$^2,~\IEEEmembership{Fellow,~IEEE}$\\
$^{1}$School of Electronics and Information Engineering, Beijing Jiaotong University, Beijing, China\\
$^2$Department of Electrical \& Computer Engineering, George Mason University, Fairfax, VA, USA
\\E-mails: \{yhuo,fanxin\}@bjtu.edu.cn, \{ywang56,ztian1\}@gmu.edu}
\maketitle

\begin{abstract}
Federated learning (FL) is an attractive paradigm for making use of rich distributed data while protecting data privacy. Nonetheless, nonideal communication links and limited transmission resources may hinder the implementation of fast and accurate FL. In this paper, we study joint optimization of communications and FL based on analog aggregation transmission in realistic wireless networks. We first derive closed-form expressions for the expected convergence rate of FL over the air, which theoretically quantify the impact of analog aggregation on FL. Based on the analytical results, we develop a joint optimization model for accurate FL implementation, which allows a parameter server to select a subset of workers and determine an appropriate power scaling factor. Since the practical setting of FL over the air encounters unobservable parameters, we reformulate the joint optimization of worker selection and power allocation using controlled approximation.
 Finally, we efficiently solve the resulting mixed-integer programming problem via a simple yet optimal finite-set search method by reducing the search space. Simulation results show that the proposed solutions developed for realistic wireless analog channels outperform a benchmark method, and achieve comparable performance of the ideal case where FL is implemented over noise-free wireless channels.
\end{abstract}

\section{Introduction}
In recent years, with the development of IoT and social networks, huge amounts of data have been generated at the edge of networks \cite{chiang2016fog}. To obtain useful information from big data, machine learning has been  widely applied to deal with complex models and tasks in emerging data-driven applications, such as autonomous driving, virtual and augmented reality \cite{park2019wireless}. Standard machine learning is usually developed under a centralized architecture, where each node located at the edge sends its collected data to a central node for centralized data processing.
However, with the exponential growth of the volume of local data and the increasing concerns on user data privacy, it is neither practical nor safe to directly transmit the data of local devices to a central node due to the limited communication and processing capability as well as the lack of privacy protection. As such, distributed machine learning is well motivated to overcome these issues. 

In the regime of distributed machine learning, federated learning (FL) has been proposed as a well noted approach for collaborative learning\cite{li2020federated}.
In FL, local workers train local models from their own data, and then transmit their local updates to a parameter server (PS). The PS aggregates these received local updates and sends the averaged update back to the local workers.
These iterative updates between the PS and workers, can be either model parameters or their gradients, for model averaging \cite{mcmahan2016communication} and gradient averaging \cite{konevcny2016federated}, respectively.
In this way, FL relieves communication overheads and protects user privacy compared to traditional data-sharing based collaborative learning, especially when the local data is in large volume and privacy-sensitive. Existing research on FL mostly focuses on FL algorithms under idealized link assumptions. However, the impacts of wireless environments on FL performance should be taken into account in the design of FL deployed in practical wireless systems. Otherwise, such impacts may introduce unwanted training errors that dramatically degrade the learning performance in terms of accuracy and convergence rate\cite{zhu2020toward}. 

To solve this problem, research efforts have been spent on optimizing network resources used for transmitting model updates in FL\cite{chen2020joint, vu2020cell}. These works of FL over wireless networks adopt digital communications, using a transmission-then-aggregation policy. Unfortunately, the communication overhead and transmission latency become large as the number of active workers increases. On the other hand, it is worth noting that FL aims for global aggregation and hence only utilizes the averaged updates of distributed workers rather than the individual local updates from workers. Alternatively, the nature of waveform superposition in wireless multiple access channel (MAC) \cite{nazer2007computation,chen2018over,goldenbaum2013harnessing,abari2015airshare} provides a direct and efficient way for transmission of the averaged updates in FL, also known as analog aggregation based FL\cite{amiri2020machine, amiri2020federated,amiri2019collaborative,zhu2019broadband, sun2019energy,yang2020federated}. As a joint transmission-and-aggregation policy, analog aggregation transmission enables all the participating workers to simultaneously upload their local model updates to the PS over the same time-frequency resources as long as the aggregated waveform represents the averaged updates, thus substantially reducing the overhead of wireless communication for FL. 

The research on analog aggregation based FL is still at early stage, leaving some fundamental questions unexplored, such as its convergence behavior and design of efficient algorithms. Given the limited transmit power and communication bandwidth at user devices, users may have to contend for communication resources when transmitting their local updates to the PS. It gives rise to the need for an efficient transmission paradigm, along with network resource allocation in terms of worker selection and transmit power control. All these practical issues motivate our work to study FL from the perspectives of both wireless communications and machine learning. In this paper, we quantify the impact of analog aggregation on the convergence behavior and performance of FL. Such quantitative results are essential in guiding the joint optimization of communication and computing resources. This paper aims at a comprehensive study on problem formulation, solution development, and algorithm implementation for the joint design and optimization of wireless communication and FL. Our key contributions are summarized as follows:

\begin{itemize}
\item We derive closed-form expressions for the expected convergence rate of FL over the air in the cases of convex and non-convex, respectively, which not only interprets but also quantifies the impact of wireless communications on the convergence and accuracy of FL over the air. Also, full-size gradient descent (GD) and mini-batched statistical gradient descent (SGD) methods are both considered in this work. These closed-form expressions unveil a fundamental connection between analog wireless communication and FL with analog aggregation, which provides a fresh perspective to measure how the parameter design of analog wireless systems affects the performance of FL over the air.
 \item Based on the closed-form theoretical results, we formulate a joint optimization problem of learning, worker selection, and power control, with a goal of minimizing the global FL loss function given limited transmit power and bandwidth. The optimization formulation turns out to be universal for the convex and non-convex cases with GD and SGD. Further, for practical implementation of the joint optimization problem in the presence of some unobservable parameters, we develop an alternative reformulation that approximates the original unattainable problem as a feasible optimization problem under the operational constraints of analog aggregation.  
    \item To efficiently solve the approximate problem,  we identity a tight solution space by exploring the relationship between the number of workers and the power scaling. Thanks to the reduced search space, we propose a simple discrete enumeration method to efficiently find the globally optimal solution. 
\end{itemize}

We evaluate the proposed joint optimization scheme for FL with analog aggregation in solving linear regression and image classification problems, respectively. Simulation results show that our proposed FL is superior to the benchmark scheme that uses random worker selection and power control, and achieves comparable performance to the ideal case where FL is implemented over noise-free wireless channels.

The remainder of this paper is organized as follows. Related work is presented in Section \ref{Sec:Related}. The system model for FL over the air and the associated joint communication and learning optimization formulation are presented in Section \ref{sec:Model}. Section \ref{sec:Convergence Analysis} derives the closed-form expressions of the expected convergence rate of the FL over the air as the foundation for algorithm design and performance analysis. Section \ref{sec:Joint optimization} provides a framework of joint optimization of communication and FL, and develops the corresponding algorithms. Section \ref{Sec:Numerical Results} presents numerical results, followed by conclusions in Section \ref{Sec:Conclusion}.


\section{Related Work}\label{Sec:Related}
This section reviews the literature and highlights the novelty of this paper with respect to related works.

To achieve communication efficiency in distributed learning, most of the existing strategies focus on digital communications, which may involve the preprocessing of transmitted updates or the management of wireless resources. For example, a popular line of work is to reduce the communication load per worker by compression of the updates under the assumptions of ideal communication links, such as exploiting coding schemes \cite{ye2018communication}, utilizing the sparsity of updates \cite{aji2017sparse}, employing quantization of the updates \cite{liu2019decentralized}, and avoiding less informative local updates via communication censoring schemes \cite{liu2019communication,8755802,chen2018lag,8646657,xu2020coke}. Another line of work is to support FL through communication resource management, such as worker scheduling schemes to maximize the number of participating workers\cite{zeng2019energy}, joint optimization of resource allocation and worker scheduling\cite{chen2020joint}, and communication and computation resource allocation and scheduling for cell-free networks\cite{vu2020cell}.

There are some pioneering works on analog aggregation based FL \cite{amiri2020machine, amiri2020federated, zhu2019broadband, amiri2019collaborative,sun2019energy,yang2020federated}, most of which focus on designing transmission schemes\cite{amiri2020machine, amiri2020federated, zhu2019broadband, amiri2019collaborative}. They adopt preselected  participating workers and fix their power allocation without further optimization along FL iterations. The optimization issues are considered in \cite{sun2019energy, yang2020federated}, but they are mainly conducted on the communication side alone, without an underlying connection to FL. When communication-based metrics are used, the optimization in existing works often suggests to maximize the number of selected workers that participate FL, but our theoretical results indicate that selecting more workers is not necessarily  better over imperfect links or under limited communication resources.  Thus, unlike these existing works, we seek to analyze the convergence behavior of analog aggregation based FL, which provides a fresh angle to interpret the specific relationship between communications and FL in the paradigm of analog aggregation. Such a connection leads to this work on a joint optimization framework for analog communications and FL, in which the work selection and power allocation decisions are optimized during the iterative FL process.

\section{System Model}\label{sec:Model}
\begin{figure}[tb]
  \centering
  \includegraphics[scale=0.6]{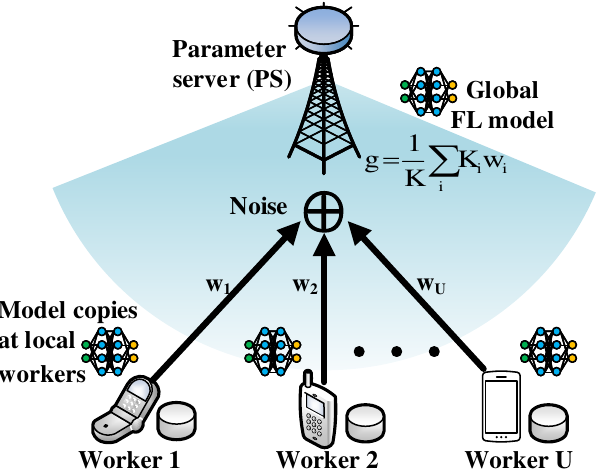}
    \caption{Federated learning via analog aggregation from wirelessly distributed data.}\label{model}
\end{figure}
As shown in Fig. \ref{model}, we consider a one-hop wireless network consisting of a single parameter server (PS) at a base station and $U$ user devices as distributed local workers. Through federated learning, the PS and all workers collaborate to train a common model for supervised learning and data inference, without sharing local data.

\subsection{FL Model}
Let $\mathcal{D}_i=\{\mathbf{x}_{i,k},\mathbf{y}_{i,k}\}_{k=1}^{K_i}$ denote the local dataset at the $i$-th worker, $i=1, \ldots, U$, where $\mathbf{x}_{i,k}$ is the input data vector, $\mathbf{y}_{i,k}$ is the labeled output vector, $k=1,2,...,K_i$, and $K_i = |\mathcal{D}_i|$ is the number of data samples available at the $i$-th worker. With $K=\sum_{i=1}^U K_i$ samples in total, these $U$ workers seek to collectively train a learning model parameterized by a global model parameter $\mathbf{w} = [w^1, \ldots, w^D] \in \mathcal{R}^D$ of dimension $D$, by minimizing the following loss function
\begin{equation}\label{eq:globallossf}
 \text{(Global loss function)}\quad F(\mathbf{w}; \mathcal{D})=\frac{1}{K}\sum_{i=1}^U\sum_{k=1}^{K_i} f(\mathbf{w};\mathbf{x}_{i,k},\mathbf{y}_{i,k}),
\end{equation}
where the global loss function $F(\mathbf{w}; \mathcal{D})$ is a summation of $K$ data-dependent components, each component $f(\mathbf{w};\mathbf{x}_{i,k},\mathbf{y}_{i,k})$ is a sample-wise local function that quantifies the model prediction error of the same data model parameterized by the shared model parameter $\mathbf{w}$, and $\mathcal{D} = \bigcup_i \mathcal{D}_i $.

In distributed learning, each worker trains a local model $\mathbf{w}_i$ from its local data $\mathcal{D}_i$, which can be viewed as a local copy of the global model $\mathbf{w}$. That is, the local loss function is
\begin{equation}\label{eq:locallossf}
 \text{(Local loss function)}\quad F_i(\mathbf{w}_i;\mathcal{D}_i)=\frac{1}{K_i}\sum_{k=1}^{K_i} f(\mathbf{w}_i;\mathbf{x}_{i,k},\mathbf{y}_{i,k}),
\end{equation}
where $\mathbf{w}_i = [w_i^1, \ldots, w_i^D] \in \mathcal{R}^D$ is the local model parameter. Through collaboration, it is desired to achieve $\mathbf{w}_i = \mathbf{w}=\mathbf{w}^*$, $\forall i$, so that all workers reach the globally optimal model $\mathbf{w}^*$. Such a distributed learning can be formulated via consensus optimization as\cite{mcmahan2016communication,wang2018cooperative}
\begin{subequations}\label{eq:optiP1}
\begin{align}\label{eq:lossfopt}
 \textbf{P1:}\quad \quad \min_{\mathbf{w}}& \quad \frac{1}{K}\sum_{i=1}^U\sum_{k=1}^{K_i} f(\mathbf{w}_i;\mathbf{x}_{i,k},y_{i,k}).
\end{align}
\end{subequations}

To solve \textbf{P1}, this paper adopts a model-averaging algorithm for FL \cite{mcmahan2016communication,wang2018cooperative}. It is essentially an iterative process consisting of both computing and communication steps at each iteration. Specifically, in each communication round, the PS broadcasts the current $\mathbf{w}$ to all workers. Then, the $i$-th worker uses a learning algorithm to update its $\mathbf{w}_i$ by minimizing its local data-dependent loss function in \eqref{eq:locallossf}. In this work, gradient descent\footnote{In this work, we take the basic gradient descent as an example, while the proposed methodology can be extended to mini-batch gradient descent as well.} is applied, in which the local model at the $i$-th local worker is updated as
\begin{align}\label{eq:localupdate0}
  \text{(Local model updating)}\quad \mathbf{w}_{i}&=\mathbf{w}-\alpha\nabla F_i(\mathbf{w}_i;\mathcal{D}_i)\nonumber\\
  &=\mathbf{w}-\frac{\alpha}{K_i}\sum_{k=1}^{K_i}\nabla f(\mathbf{w};\mathbf{x}_{i,k},\mathbf{y}_{i,k}),
\end{align}
where $\alpha$ is the learning rate, and $\nabla f(\mathbf{w};\mathbf{x}_{i,k},\mathbf{y}_{i,k})$ is the gradient of $f(\mathbf{w};\mathbf{x}_{i,k},\mathbf{y}_{i,k})$ with respect to $\mathbf{w}$.

When local updating is completed, each worker transmits its updated parameter $\mathbf{w}_i$ to the PS via wireless uplinks to update the global $\mathbf{w}$ as
\begin{equation}\label{eq:g}
 \text{(Global model updating)}\quad \mathbf{w}=\frac{\sum_{i=1}^U K_i\mathbf{w}_{i}}{K}.\qquad\qquad\ \
\end{equation}
Then, the PS broadcasts $\mathbf{w}$ in \eqref{eq:g} to all participating workers as their initial value in the next round. The FL implements the local model-updating in \eqref{eq:localupdate0} and the global model-averaging in \eqref{eq:g} iteratively, until convergence. It has been shown that this FL algorithm converges to the globally optimal solution of the original problem in \textbf{P1} under the conditions that $F$ is a convex function and the data transmission between the PS and workers is error-free\cite{mcmahan2016communication,wang2018cooperative}.

Note that the implementation steps in \eqref{eq:localupdate0} and \eqref{eq:g} only concern the computational aspect of FL, by assuming perfect communications for both the global $\mathbf{w}$ and local $\mathbf{w}_i$ between the PS and all workers. However, the communication impacts on FL performance should not be ignored. Especially in practical wireless network environments, certain errors are inevitably introduced during transmissions of the updates due to the imperfect characteristics of wireless channels.

\subsection{Analog Aggregation Transmission Model}
To avoid heavy communication overhead and save transmission bandwidth of FL over wireless
channels, we adopt analog aggregation without coding, which allows multiple workers to simultaneously upload their local model updates to the PS over the same time-frequency resources.
All workers transmit their local $\mathbf{w}_i$'s in an analog form with perfect time synchronization among them\footnote{The implementation of time synchronization and the impact of imperfect synchronization are beyond the scope of this work. Interested readers are referred to \cite{abari2015airshare, goldenbaum2013robust}.}. In this way, the local updates $\mathbf{w}_i$'s are aggregated over the air to implement the global model updating step in \eqref{eq:g}. Such an analog aggregation is conducted in an entry-wise manner. That is, the $d$-th entries $w_i^d$ from all workers, $i=1, ..., U$, are aggregated to compute $w^d$ in \eqref{eq:g}, for any $d\in [1, D]$.

Let $\mathbf{p}_{i,t}=[p^1_{i,t},\dots,p^d_{i,t},\ldots,p^D_{i,t}]$ denote the power control vector of worker $i$ at the $t$-th iteration. Noticeably, the choice of $\mathbf{p}_{i,t}$ in FL over the air should be made not only to effectively implement the aggregation rule in \eqref{eq:g}, but also to properly accommodate the need for network resource allocation. Accordingly, we set the power control policy as
\begin{equation}\label{power_control}
 p^d_{i,t}=\frac{\beta^d_{i,t}K_ib^d_{t}}{h^d_{i,t}},
\end{equation}
where $h_{i,t}$ is the channel gain between the $i$-th worker and the PS at the $t$-th iteration\footnote{In this paper, we assume the channel state information (CSI) to be constant within each iteration, but may vary over iterations. We also assume that the CSI is perfectly known at the PS, and leave the imperfect CSI case in future work.}, $b^d_{t}$ is the power scaling factor, and $\beta^d_{i,t}$ is a transmission scheduling indicator. That is, $\beta^d_{i,t}=1$ means that the $d$-th entry of the local model parameter $\mathbf{w}_{i,t}$ of the $i$-th worker is scheduled to contribute to the FL algorithm at the $t$-th iteration, and $\beta^d_{i,t}=0$, otherwise. Through power scaling, the transmit power used for uploading the $d$-th entry from the $i$-th worker should  not exceed a maximum power limit $P_i^{d,\max} = P_i^{\max}$ for any $d$, as follows: 

\begin{equation}\label{power_limitation}
 |p^d_{i,t}w^d_{i,t} |^2=\left|\frac{\beta^d_{i,t}K_ib^d_{t}}{h^d_{i,t}}w^d_{i,t}\right|^2\leq P_i^{\max}.
\end{equation}

At the PS side, the received signal at the $t$-th iteration can be written as
\begin{align}\label{4}
  \mathbf{y}_{t}&=\sum_{i=1}^U \mathbf{p}_{i,t}\odot\mathbf{w}_{i,t}\odot\mathbf{h}_{i,t}+\mathbf{z}_{t}  \nonumber
\\ &=\sum_{i=1}^U K_i\mathbf{b}_{t}\odot\bm{\beta}_{i,t}\odot\mathbf{w}_{i,t}+\mathbf{z}_{t},
 \end{align}
where $\odot$ represents Hadamard product, $\mathbf{h}_{i,t}=[h^1_{i,t},h^2_{i,t},...,h^D_{i,t}]$, $\bm{\beta}_{i,t}=[\beta^1_{i,t},\beta^2_{i,t},..,\beta^D_{i,t}]$, $\mathbf{b}_{t}=[b^1_{t}, b^2_{t}, ..., b^D_{t}]$, and $\mathbf{z}_{t} \sim \mathcal{CN}(0,\sigma^2\mathbf{I})$ is additive white Gaussian noise (AWGN).

Given the received $\mathbf{y}_{t}$, the PS estimates $\mathbf{w}_{t}$ via a post-processing operation as
\begin{align}\label{eq:gt}
  \mathbf{w}_{t}=  &\left(\sum_{i=1}^U K_i\bm{\beta}_{i,t}\odot\mathbf{b}_{t}\right)^{\odot-1}\odot\mathbf{y}_{t}\nonumber\\
  =&\left(\sum_{i=1}^U K_i\bm{\beta}_{i,t}\right)^{\odot-1}\sum_{i=1}^U K_i \bm{\beta}_{i,t} \odot\mathbf{w}_{i,t}  +\left(\sum_{i=1}^U K_i\bm{\beta}_{i,t}\odot\bm{b}_{t}\right)^{\odot-1}\odot\mathbf{z}_{t},
\end{align}
where $(\sum_{i=1}^U K_i\bm{\beta}_{i,t}\odot\mathbf{b}_{t})^{\odot-1}$ is a properly chosen scaling vector to produce equal weighting for participating $\mathbf{w}_i$'s in \eqref{eq:gt} as desired in \eqref{eq:g}, and $(\mathbf{X})^{\odot-1}$ represents the inverse Hadamard operation of $\mathbf{X}$ that calculates its entry-wise reciprocal. Noticeably, in order to implement the averaging of \eqref{eq:g} in FL over the air, such a post-processing operation requires $\mathbf{b}_{t}$ to be the same for all workers for given $t$ and $d$, which allows to eliminate $\mathbf{b}_{t}$ from the first term in \eqref{eq:gt}.

Comparing \eqref{eq:gt} with \eqref{eq:g}, there exist differences between $\mathbf{w}_{t}$ and $\mathbf{w}$ due to the effect of wireless communications. This work aims to mitigate such a gap through optimizing the worker selection $\bm{\beta}_{i,t}$ and power scaling $\mathbf{b}_{t}$ for FL over the air. To this end, our next step is to unveil an important but unexplored foundation, i.e., how wireless communications affect the convergence behavior of FL over the air.

\section{The Convergence Analysis of FL with Analog Aggregation}\label{sec:Convergence Analysis}

In this section, we study the effect of analog aggregation transmission on FL over the air, by analyzing its convergence behavior for both the convex and the non-convex cases. To average the effects of instantaneous SNRs, we derive the expected convergence rate of FL over the air, which quantifies the impact of wireless communications on FL using analog aggregation transmissions.

\subsection{Convex Case}
We first make the following assumptions that are commonly adopted in the optimization literature \cite{shamir2014communication,magnusson2020maintaining,Bertsekas1996Neuro,friedlander2012hybrid,chen2020joint,alistarh2018convergence}.

\textbf{Assumption 1 (Lipschitz continuity, smoothness):} The gradient $\nabla F(\mathbf{w})$ of the loss function $F(\mathbf{w})$ is uniformly Lipschitz continuous with respect to $\mathbf{w}$, that is,
\begin{eqnarray}\label{eq:Lipschitz}
\|\nabla F(\mathbf{w}_{t+1})-\nabla F(\mathbf{w}_{t})\|\leq L\|\mathbf{w}_{t+1}-\mathbf{w}_{t}\|, \quad \forall \mathbf{w}_{t}, \mathbf{w}_{t+1},
\end{eqnarray}
where $L$ is a positive constant, referred to as a Lipschitz constant for the function $F(\cdot)$.

\textbf{Assumption 2 (strongly convex):} $\nabla F(\mathbf{w})$ is strongly convex with a positive parameter $\mu$, obeying 
\begin{align}\label{eq:stronglyconvex}
F(\mathbf{w}_{t+1})\geq F(\mathbf{w}_{t})+(\mathbf{w}_{t+1}-\mathbf{w}_{t})^T\nabla F(\mathbf{w}_{t})+\frac{\mu}{2}\|\mathbf{w}_{t+1}-\mathbf{w}_{t}\|^2, \quad \forall \mathbf{w}_{t}, \mathbf{w}_{t+1}.
\end{align}


\textbf{Assumption 3 (bounded local gradients):} The sample-wised local gradients at local workers are bounded by their global counterpart\cite{Bertsekas1996Neuro,friedlander2012hybrid}
\begin{eqnarray}\label{eq:bound}
\| \nabla f(\mathbf{w}_{t})\|^2 \leq \rho_1+\rho_2\| \nabla F(\mathbf{w}_{t})\|^2,
\end{eqnarray}
where $\rho_1, \rho_2\geq0$.

According to \cite{konevcny2016federated, yuan2016convergence}, the FL algorithm applied over ideal wireless channels is able to solve \textbf{P1} and converges to an optimal $\mathbf{w}^*$. In the presence of wireless transmission errors, we derive the expected convergence rate of the FL over the air with analog aggregation, as in \textbf{Theorem \ref{Theorem1}}.

\begin{theorem}\label{Theorem1}
Adopt \textbf{Assumptions 1-3}, and denote the globally optimal learning model in \eqref{eq:optiP1} as $\mathbf{w}^*$. The model updating rule for $\mathbf{w}_{t}$ of the FL-over-the-air scheme is given by \eqref{eq:gt}, $\forall t$. Given the transmit power scaling factors $\bm{b}_{t}$, worker selection vectors $\bm{\beta}_{i,t}$, and setting the learning rate to be $\alpha = \frac{1}{L}$, the expected performance gap $\mathbb{E}[F(\mathbf{w}_{t})-F(\mathbf{w}^*)]$ of $\mathbf{w}_{t}$ at the $t$-th iteration is given by
\begin{align}\label{eq:Theorem1}
\mathbb{E}[F(\mathbf{w}_{t})-F(\mathbf{w}^*)]
\leq B_{t}+A_{t}\mathbb{E}[F(\mathbf{w}_{t-1})-F(\mathbf{w}^*)],
\end{align}
with
\begin{align}\label{eq:A}
A_{t}=1-\frac{\mu}{L}+\rho_2\sum_{d=1}^D\Bigg(\frac{K}{\sum_{i=1}^U K_i\beta^d_{i,t}}-1\Bigg),
\end{align}
\begin{align}\label{eq:B}
B_{t}&=\frac{\rho_1}{2L}\sum_{d=1}^D\Bigg(\frac{K}{\sum_{i=1}^U K_i\beta^d_{i,t}}-1\Bigg)+\Bigg\|\left(\sum_{i=1}^U K_i\bm{\beta}_{i,t}\odot\bm{b}_{t}\right)^{\odot-1}\Bigg\|^2\frac{L\sigma^2}{2},
\end{align}
where the expectation is over the channel AWGN of zero mean and variance $\sigma^2$.
\end{theorem}
\begin{proof}
The proof of \textbf{Theorem \ref{Theorem1}} is provide in Appendix \ref{Appendix A}.
\end{proof}

Based on \textbf{Theorem \ref{Theorem1}}, we further derive the cumulative performance gap resulted from wireless communications and the worker selection of the whole FL process, summarized by the following \textbf{Lemma \ref{Lemma1}}.
\begin{lemma}\label{Lemma1}
Given an initial global model $\mathbf{w}_{0}$, the cumulative performance gap $\mathbb{E}[F(\mathbf{w}_{t})-F(\mathbf{w}^*)]$ of FL after $t$ iterations is bounded by
\begin{align}\label{eq:Lemma1}
\mathbb{E}[F(\mathbf{w}_{t})-F(\mathbf{w}^*)]
\leq& \underbrace{\sum_{i=1}^{t-1}\prod_{j=1}^i A_{t+1-j}B_{t-i}+B_{t}}_{\Delta_{t}}
+\prod_{j=1}^{t}A_{j}\mathbb{E}[F(\mathbf{w}_{0})-F(\mathbf{w}^*)].
\end{align}
\end{lemma}
\begin{proof}
Given the expected performance gap at the $t$-th iteration in \eqref{eq:Theorem1}, we carry out recursions as follows:
\begin{align}\label{23}
\mathbb{E}[F(\mathbf{w}_{t})-F(\mathbf{w}^*)]
&\leq B_{t}+A_{t}\mathbb{E}[F(\mathbf{w}_{t-1})-F(\mathbf{w}^*)]\nonumber\\ \nonumber
&\leq B_{t}+A_{t}\bigg(B_{t-1}+A_{t-1}\mathbb{E}[F(\mathbf{w}_{t-2})-F(\mathbf{w}^*)]\bigg)
\\ \nonumber
&\leq...\nonumber\\
&\leq\sum_{i=1}^{t-1}\prod_{j=1}^i A_{t+1-j}B_{t-i}+B_{t}+\prod_{j=1}^{t}A_{j}\mathbb{E}[F(\mathbf{w}_{0})-F(\mathbf{w}^*)].
\end{align}.
\end{proof}

\textbf{Lemma \ref{Lemma1}} reveals that the FL algorithm converges asymptotically in $t$ under mild conditions, as stated in \textbf{Proposition \ref{Propositon1}}.

\begin{prop}\label{Propositon1}
Given the learning rate $\alpha=\frac{1}{L}$, the convergence of the FL algorithm is guaranteed with $\lim_{t\rightarrow \infty} \mathbf{w}_t = \mathbf{w}^*$, as long as $\rho_2$ in \eqref{eq:bound} satisfies the following condition:
\begin{align}\label{conditionconver}
0<\rho_2< \frac{\mu}{(\frac{K}{K_{min}}-1)DL},
\end{align}
where $K_{min}=\min\{K_i\}_{i=1}^U$.
\end{prop}
\begin{proof}
When $A_{t}<1, \ \forall t$, it is evident that  $\lim_{t\rightarrow\infty}\prod_{j=1}^{t+1}A_{j}=0$. From \textbf{Lemma \ref{Lemma1}}, to guarantee the convergence, a sufficient condition is to ensure $A_{max}\triangleq\max\{A_t, t=1,2...\}< 1$. Given \eqref{eq:A}, it holds that
\begin{align}\label{eq:Amax}
A_{t}&=1-\frac{\mu}{L}+\rho_2\sum_{d=1}^D\Bigg(\frac{K}{\sum_{i=1}^U K_i\beta^d_{i,t}}-1\Bigg)\nonumber\\
&\leq 1-\frac{\mu}{L}+\rho_2\sum_{d=1}^D\Bigg(\frac{K}{K_{min}}-1\Bigg),
\end{align}
where $K_{min}=\min\{K_i\}_{i=1}^U$. When all workers have the same amount of data, i.e., $K_i=\frac{K}{U}, \forall i$, then $A_{t}\leq 1-\frac{\mu}{L}+\rho_2 D(U-1)$.

To ensure $A_{max}< 1$, we have
\begin{align}\label{eq:Amaxrou}
A_{max}\leq 1-\frac{\mu}{L}+\rho_2\sum_{d=1}^D\Bigg(\frac{K}{K_{min}}-1\Bigg)< 1.
\end{align}
From \eqref{eq:Amaxrou}, it holds that $\rho_2 < \frac{\mu}{(\frac{K}{K_{min}}-1)DL}$.  On the other hand, $\rho_2 > 0$, according to \eqref{eq:bound} in \textbf{Assumption 3}. As a result, we have $0<\rho_2< \frac{\mu}{(\frac{K}{K_{min}}-1)DL}$, which completes the proof.
\end{proof}

\textbf{Proposition \ref{Propositon1}} indicates that the convergence behavior of the FL algorithm depends on both the learning-related parameters, i.e., $\mu,L,\rho_1,\rho_2$, and communication-related parameters, including $\bm{\beta}$, $\mathbf{b}$ and $\sigma^2$. Interestingly, the channel noise $\sigma^2$ and $\mathbf{b}$ do not affect $A_t$, and hence they do not affect the convergence of the FL algorithm but determine the steady state that the FL algorithm converges to. 

\textbf{Lemma \ref{Lemma1}} also provides the expected convergence rate of an FL algorithm when the transmission link is error-free. In this ideal case, it offers the fastest convergence rate achievable by the FL algorithm, which is derived by the following \textbf{Lemma \ref{Lemma}}.

\begin{lemma}\label{Lemma}
Consider a resource-unconstrained and error-free mode where the effects of wireless channels, as well as that of noise, are already mitigated or fully compensated.
Given the optimal global $\mathbf{w}^*$ and the learned $\mathbf{w}_{t}$ in \eqref{eq:gt}, the upper bound of $\mathbb{E}[F(\mathbf{w}_{t})-F(\mathbf{w}^*)]$ for the FL over the air is given by
\begin{align}\label{eq:Lemma}
\mathbb{E}[F(\mathbf{w}_{t})-F(\mathbf{w}^*)]\leq
(1-\frac{\mu}{L})^{t}\mathbb{E}[F(\mathbf{w}_{0})-F(\mathbf{w}^*)].
\end{align}
\end{lemma}
\begin{proof}
  Without channel noise or worker selection (that is, all workers participate the FL and deliver their data perfectly), we have $\sigma^2=0$ and $\sum_{d=1}^D\Big(\frac{K}{\sum_{i=1}^U K_i\beta^d_{i,t}}-1\Big)=0$. Then, in \eqref{eq:A} and \eqref{eq:B}, we have $B_t=B=0$ and $A_{t}=A=1-\frac{\mu}{L}, \ \forall t$.  As a result, \eqref{eq:Lemma1} is reduced to \eqref{eq:Lemma}.
\end{proof}

It is worth noting that \textbf{Lemma \ref{Lemma}} provides the convergence rate for the ideal case, which assumes that the impacts of wireless communications, including noise, channel and constrained resources, are all mitigated to result in error-free transmission. According to \eqref{eq:Lemma1} in the realistic case, the trajectory of $\mathbb{E}[F(\mathbf{w}_{t+1})]$ exhibits jump discontinuity with a gap term $\Delta_t$ at each step $t$, as defined in \eqref{eq:Lemma1}:
\begin{align*}
\Delta_t=\sum_{i=1}^{t-1}\prod_{j=1}^i A_{t+1-j}B_{t-i}+B_{t}.
\end{align*}
This gap reflects the impact of wireless communication factors on FL, by means of the worker selection, transmit power scaling and AWGN. Intuitively, this gap diminishes as the number of selected workers increases, which reduces the value of $A_{t}$. Meanwhile, as the power scaling factor $\bm{b}_{t}$ increases, $B_{t}$ is decreased, which leads to a reduction of the gap as well.  Hence, it is necessary to optimize transmit power scaling factors and worker selection in order to minimize the gap in \eqref{eq:Lemma1} for the implementation of FL algorithms over a realistic wireless network.

\subsection{Non-convex Case}



When the loss function $F(w)$ is nonconvex, such as in the case of convolutional neural networks,
we derive the convergence rate of the FL over the air with analog aggregation for the nonconvex case without \textbf{Assumption 2}, which is summarized in \textbf{Theorem \ref{Theoremnonconvex}}.


\begin{theorem}\label{Theoremnonconvex}
Under the \textbf{Assumptions 1} and \textbf{3} for the non-convex case, given the transmit power scaling factors $\bm{b}_{t}$, worker selection vectors $\bm{\beta}_{i,t}$, the optimal global FL model $\mathbf{w}^*$, and the learning rate $\alpha = \frac{1}{L}$, the convergence at the $T$-th iteration is given by
\begin{align}\label{eq:Egt1nonctheorem}
\frac{1}{T}\sum_{t=1}^T\|\nabla F(\mathbf{w}_{t-1})\|^2\leq&\frac{2L }{T(1-\rho_2D(\frac{K}{K_{min}}-1))}\mathbb{E}[F(\mathbf{w}_{0})]-F(\mathbf{w}^*)]\nonumber\\&+\frac{2L\sum_{t=1}^TB_t }{T(1-\rho_2D(\frac{K}{K_{min}}-1))}
.
\end{align}
\end{theorem}
\begin{proof}
Please refer to Appendix \ref{Appendix B}.
\end{proof}

As we can see from \textbf{Theorem \ref{Theoremnonconvex}}, when $T$ is large enough, we have
\begin{align}\label{eq:Egt1nonctheorem}
\min_{0,1,...,T}\mathbb{E}[\|\nabla F(\mathbf{w}_{t-1})\|^2]  \leq \frac{1}{T}\sum_{t=1}^T\|\nabla F(\mathbf{w}_{t-1})\|^2\overset{T\rightarrow \infty}{\leq}\underbrace{\frac{2L\sum_{t=1}^TB_t }{T(1-\rho_2D(\frac{K}{K_{min}}-1))}}_{\bigtriangleup_T^{NC}},
\end{align}
which guarantees convergence of the FL algorithm to a stationary point\cite{bottou2018optimization,wang2018cooperative}. Similarly, the performance gap at the step $t$ for non-convex cases is given by
\begin{align}\label{eq:gapforNC}
\bigtriangleup_t^{NC}=\frac{2L\sum_{t=1}^TB_t }{T(1-\rho_2D(\frac{K}{K_{min}}-1))}.
\end{align}

Note that the non-convex case and the convex shares the same sufficient condition for convergence as \eqref{conditionconver} in \textbf{Proposition \ref{Propositon1}}.

\subsection{Stochastic gradient descent}
Our work can be extended to stochastic versions of gradient descent (SGD) as well. Here, we provide convergence analysis for mini-batch gradient descent with a constant mini-batch size $K_b$, while the results directly apply to the standard SGD by setting $K_b = 1$.
Theorem 3 summarizes the convergence behavior of SGD for the strongly convex case.


\begin{theorem}\label{SGDtheorem}
Under the \textbf{Assumptions 1, 2} and \textbf{3} for the convex case, and given the transmit power scaling factors $\bm{b}_{t}$, worker selection vectors $\bm{\beta}_{i,t}$, the optimal global FL model $\mathbf{w}^*$, the learning rate $\alpha = \frac{1}{L}$ and the mini-batch size $K_b$, the convergence behavior of the SGD implementation of FL over the air is given by
\begin{align}\label{eq:SGDtheorem}
\mathbb{E}[F(\mathbf{w}_{t})-F(\mathbf{w}^*)]
\leq& \underbrace{\sum_{i=1}^{t-1}\prod_{j=1}^i A^{SGD}_{t+1-j}B^{SGD}_{t-i}+B^{SGD}_{t}}_{\Delta^{SGD}_t}
+\prod_{j=1}^{t}A^{SGD}_{j}\mathbb{E}[F(\mathbf{w}_{0})-F(\mathbf{w}^*)],
\end{align}
where
\begin{align}
A^{SGD}_{t}=&1-\frac{\mu}{L}+\rho_2\Bigg(\sum_{d=1}^{D}\Bigg(\frac{(\sum_{i=1}^UK_b)^2 -2K(\sum_{i=1}^UK_b)}{K^2}+\frac{(\sum_{i=1}^UK_b)}{\sum_{i=1}^U K_b\beta^d_{i,t}}\Bigg)\nonumber\\ &+\frac{(\sum_{i=1}^{U}(K_i-K_b))^2}{K^2}\Bigg),\label{eq:ASGD}
\end{align}
\begin{align}
B^{SGD}_{t}=&\frac{\rho_1}{2L}\Bigg(\sum_{d=1}^{D}\Bigg(\frac{(\sum_{i=1}^UK_b)^2 -2K(\sum_{i=1}^UK_b)}{K^2}+\frac{(\sum_{i=1}^UK_b)}{\sum_{i=1}^U K_b\beta^d_{i,t}}\Bigg)\nonumber\\&+\frac{(\sum_{i=1}^{U}(K_i-K_b))^2}{K^2}\Bigg) +\left\|\left(\sum_{i=1}^U K_i\bm{\beta}_{i,t}\odot\bm{b}_{t}\right)^{\odot-1}\right\|^2\frac{L\sigma^2}{2}.\label{eq:BSGD}
\end{align}
\end{theorem}
\begin{proof}
Please refer to Appendix \ref{Appendix C}.
\end{proof}

From \textbf{Theorem \ref{SGDtheorem}}, the cumulative performance gap of FL after the $t$-th iteration for the SGD case is bounded by
\begin{align}\label{eq:gapforSGD}
\bigtriangleup_t^{SGD}=\sum_{i=1}^{t-1}\prod_{j=1}^i A^{SGD}_{t+1-j}B^{SGD}_{t-i}+B^{SGD}_{t}.
\end{align}

\begin{remark}
If $K_b$ is set to be $K_i$, then \textbf{Theorem \ref{SGDtheorem}} for SGD is the same as \textbf{Theorem \ref{Theorem1}} for GD. In addition, since the common mini-batch size is no larger than the minimum local data size, i.e., $K_b\leq K_{min}\leq \frac{K}{U}$, both $A_t^{SGD}$ in \eqref{eq:ASGD} and $B^{SGD}_t$ in \eqref{eq:BSGD} decrease as $K_b$ increases, which leads to a smaller $\Delta_t^{SGD}$ in (25). In other words, FL has a better convergence performance with a larger $K_b$. On the other hand, such an improvement on performance is achieved at the cost of high computation load at the local workers in each communication round, which reflects the tradeoff between training performance and computational complexity in SGD. 
\end{remark}

Similarly, we can also derive the mild convergence condition for SGD, given by the following \textbf{Proposition \ref{Proposition2}}.
\begin{prop}\label{Proposition2}
Given the learning rate $\alpha=\frac{1}{L}$, the convergence of the FL algorithm for SGD cases  is guaranteed with $\lim_{t\rightarrow \infty} \mathbf{w}_t = \mathbf{w}^*$, as long as $\rho_2$ in \eqref{eq:bound} satisfies the following condition:
\begin{align}\label{24SGD}
0<\rho_2< \frac{\mu}{(\frac{2UK_b}{K}+\frac{U^2K^2_b}{K^2}+DU-\frac{2DUK_b}{K}+\frac{DU^2K_b^2}{K^2})L}.
\end{align}
\end{prop}
\begin{proof}
Similar to the GD case, to guarantee the convergence, a sufficient condition is still to ensure $A^{SGD}_{max}\triangleq\max\{A^{SGD}_t, t=1,2...\}< 1$. It holds that,
\begin{align}\label{SGDAthold}
A^{SGD}_{t}=&1-\frac{\mu}{L}+\rho_2\Bigg(\frac{(K-UK_b)^2}{K^2} +\sum_{d=1}^{D}\Bigg(\frac{U^2K_b^2-2KUK_b}{K^2}+\frac{U}{\sum_{i=1}^U \beta^d_{i,t}}\Bigg)\Bigg)
\nonumber\\\leq& 1-\frac{\mu}{L}+\rho_2\Bigg(\frac{(K-UK_b)^2}{K^2} +\frac{DU^2K_b^2-2DKUK_b}{K^2}+DU\Bigg),
\end{align}
which comes from the fact that $\sum_{i=1}^U \beta^d_{i,t}\geq 1$.

Thus, we have
\begin{align}\label{SGDAmax}
A^{SGD}_{max}\leq 1-\frac{\mu}{L}+\rho_2\Bigg(\frac{(K-UK_b)^2+DU^2K_b^2-2DKUK_b+DUK^2}{K^2}\Bigg)< 1.
\end{align}

From \eqref{SGDAmax}, we can get $\rho_2< \frac{\mu}{(1-\frac{2UK_b}{K}+\frac{U^2K^2_b}{K^2}+DU-\frac{2DUK_b}{K}+\frac{DU^2K_b^2}{K^2})L}$. Considering $\rho_2>0$ in \textbf{Assumption 3}, we get \eqref{24SGD} as a result.
\end{proof}

\section{Performance Optimization for Federated Learning over the Air}\label{sec:Joint optimization}
In this section, we first formulate a joint optimization problem to reduce the gap for FL over the air, which turns out to be applicable for both the convex and non-convex cases, using either GD or SGD implementations. To make it applicable in practice in the presence of some unobservable parameters at the PS, we reformulate it to an approximate problem by imposing a conservative power constraint. To efficiently solve such an approximate problem, we first identify a tight solution space and then develop an optimal solution via discrete programming.

\subsection{Problem Formulation for Joint Optimization}

Since we are concerned with convergence accuracy, our optimization problem boils down to minimizing the performance gap for different cases (i.e., $\bigtriangleup_t$, $\bigtriangleup_t^{NC}$, and $\bigtriangleup_t^{SGD}$) at each iteration under the corresponding convergence conditions (i.e., \textbf{Proposition \ref{Propositon1}} and \textbf{Proposition \ref{Proposition2}}).

We recognize that solving \textbf{P1} amounts to iteratively minimizing those gap $\bigtriangleup_t$, $\bigtriangleup_t^{NC}$, and $\bigtriangleup_t^{SGD}$ under the transmit power constraint in \eqref{power_limitation}. At the $t$-th iteration, the objective functions for those three cases are given by
\begin{align}\label{eq:detGD}
\bigtriangleup_t=&B_t+A_t\bigtriangleup_{t-1},\\
\bigtriangleup^{NC}_t=&B_t,\label{eq:detNC}\\
\bigtriangleup^{SGD}_t=&B^{SGD}_t+A^{SGD}_t\bigtriangleup^{SGD}_{t-1}.\label{eq:detSGD}
\end{align}
where $\bigtriangleup_0=0$ and $\bigtriangleup^{SGD}_0=0$.
Note that when the optimization is executed at the $t$-th iteration, $\bigtriangleup_{t-1}$ and $\bigtriangleup^{SGD}_{t-1}$ can be treated as constants.

Considering the entry-wise transmission for analog aggregation, we remove irrelevant items and extract the component of the $d$-th entry from those gap in \eqref{eq:detGD}, \eqref{eq:detNC} and \eqref{eq:detSGD} as the objective to minimize, which is given by
\begin{align}\label{eq:entryGD}
R_t[d]=&\frac{L\sigma^2}{2\left(\sum_{i=1}^U\beta^d_{i,t}K_ib^d_{t}\right)^2}+\frac{K\rho_1+2KL\rho_2\bigtriangleup_{t-1}}{2L\sum_{i=1}^U K_i\beta^d_{i,t}}, \quad \forall d,\\
R_t^{NC}[d]=&\frac{L\sigma^2}{2\left(\sum_{i=1}^U\beta^d_{i,t}K_ib^d_{t}\right)^2}+\frac{K\rho_1}{2L\sum_{i=1}^U K_i\beta^d_{i,t}}, \quad \forall d,\label{eq:entryNC}\\
R_t^{SGD}[d]=&\frac{L\sigma^2}{2\left(\sum_{i=1}^U\beta^d_{i,t}K_ib^d_{t}\right)^2} +\frac{U(\rho_1+2L\rho_2\bigtriangleup_{t-1})}{2L\sum_{i=1}^U K_i\beta^d_{i,t}}, \quad \forall d.\label{eq:entrySGD}
\end{align}

Since all entries indexed by $d$ are separable with respect to the design parameters, we perform entry-wise optimization by considering $\mathbf{w}_{t}$ and $\mathbf{w}_{i,t}$  one entry at a time, where the superscript $d$ and the index of different cases are omitted hereafter. To determine the worker selection vector $\beta_{i,t}$ and the power scaling factor $b_t$ at the $t$-th iteration, the PS carries out a joint optimization problem formulated as follows:
\begin{subequations}\label{IterationOpt}
\begin{align}
\textbf{P2:}\quad \min_{\{b_{t},\beta_{i,t}\}_{i=1}^{U}} & R_{t}\\ \label{con:pmax}
\text{s.t.} & \  \bigg|\frac{\beta_{i,t}K_ib_{t}}{h_{i,t}}w_{i,t}\bigg|^2\leq P_i^{\max}, \\ \nonumber
& \ {\beta}_{i,t}\in\{0,1\}, i\in \{1,2,...,U\},
\end{align}
\end{subequations}
where $K_i$ should be $K_b$ in \eqref{con:pmax} for the SGD case. 

However, in \eqref{con:pmax}, the knowledge of $\{w_{i,t}\}_{i=1}^U$ is needed but is unavailable to the PS due to analog aggregation. 

To overcome this issue, we reformulate a practical optimization problem via an approximation that $\mathbf{w}_{t-1}\approx \frac{1}{U}\sum_{i=1}^U \mathbf{w}_{i,t}$, in light of the consensus constraint in \textbf{P1}.  According to \eqref{eq:localupdate} in the proof of \textbf{Theorem \ref{Theorem1}}, each local parameter $\mathbf{w}_{i,t}$ is updated from the broadcast $\mathbf{w}_{t-1}$ along the direction of the averaged gradient over its local data $\frac{\alpha}{K_i}\sum_{k=1}^{K_i}\nabla f(\mathbf{w}_{t-1};\mathbf{x}_{i,k},\mathbf{y}_{i,k})$. Hence, it is reasonable to make the following common assumption on bounded local gradients, considering that the local gradients can be controlled by adjusting the learning rate or through simple clipping \cite{liu2019decentralized,wang2018cooperative,stich2018sparsified,tang2019doublesqueeze}.

\textbf{Assumption 4 (bounded local gradients):} The gap between the global parameter $w_{t-1}$ and the local parameter update $w_{i,t},\forall i, t$ is bounded by
\begin{align}\label{eq:updatebound}
| w_{t-1}-w_{i,t}| \leq \eta,
\end{align}
where $\eta\geq 0$ is related to the learning rate $\alpha$ that satisfies the following condition\footnote{This implies the value range of $\eta$. In practice, $\eta$ can take $|w_{t-1}-w_{t-2}|$. In addition, for the SGD case, we have $\eta\geq \max\{\{|\alpha \mathbb{E}_{\mathcal{D}_i}[\nabla f(w,\mathbf{x}_{i,k},\mathbf{y}_{i,k})]|\}_{i=1}^U\}$}
\begin{align}\label{eq:updateboundcondition}
\eta\geq \max\left\{\left\{\Bigg|\frac{\alpha}{K_i}\sum_{k=1}^{K_i}\nabla f(w,\mathbf{x}_{i,k},\mathbf{y}_{i,k})\Bigg|\right\}_{i=1}^U\right\}.
\end{align}

Under \textbf{Assumption 4}, we reformulate the original optimization problem \textbf{(P2)} into the following problem (\textbf{P3}), by replacing $w_{i, t}$ in \eqref{con:pmax} by its approximation: 
%
\begin{subequations}
\begin{align}\label{IterationOpttighten}
\textbf{P3:}\quad \min_{\{b_{t},\beta_{i,t}\}_{i=1}^{U}} & R_{t}\\ \label{constraint:P3b} 
\text{s.t.} & \  \bigg|\frac{\beta_{i,t}K_ib_{t}}{h_{i,t}}\bigg|^2(|w_{t-1}|+\eta)^2\leq P_i^{\max}, \\\label{constraint:P3c}
& \ {\beta}_{i,t}\in\{0,1\}, i\in \{1,2,...,U\},
\end{align}
\end{subequations}
where the power constraint \eqref{constraint:P3b} is constructed based on the fact that
\begin{align}\label{Iteratio}
\bigg|\frac{\beta_{i,t}K_ib_{t}}{h_{i,t}}w_{i,t}\bigg|^2 =&\bigg|\frac{\beta_{i,t}K_ib_{t}}{h_{i,t}}\bigg|^2|w_{i,t}|^2\nonumber\\
\leq&\bigg|\frac{\beta_{i,t}K_ib_{t}}{h_{i,t}}\bigg|^2(|w_{t-1}|+\eta)^2.
\end{align}

Since $w_{t-1}$ is always available at the PS, \textbf{P3} becomes a feasible formulation for adoption in practice. Next, we develop the optimal solution to \textbf{P3}.

\subsection{Optimal Solution to \textbf{P3} via Discrete Programming}
At first glance, a direct solution to \textbf{P3} leads to a mixed integer programming (MIP), which unfortunately incurs high complexity. To solve \textbf{P3} in an efficient manner, we develop a simple solution by identifying a tight search space without loss of optimality. The tight search space, given in the following \textbf{Theorem \ref{Theorem_SolutionSpace}}, is a result of the constraints in \eqref{constraint:P3b} and \eqref{constraint:P3c}, irrespective of the objective function \eqref{IterationOpttighten}. Hence, it holds universally for any $R_t$, such as those in \eqref{eq:entryGD}-\eqref{eq:entrySGD}.

\begin{theorem}\label{Theorem_SolutionSpace}
When all the required parameters in \textbf{P3} i.e., $\{P_{i}^{\max}, w_{t-1}, h_{i,t}, K_i, \eta\}_{i=1}^U$, are available at the PS, the solution space of $(b_t, \beta_{i,t})$ in \textbf{P3} can be reduced to the following tight search space without loss of optimality as
  \begin{align}\label{space}
\mathcal{S}=\left\{\bigg\{\left(b_t^{(k)},\beta_{i,t}^{(k)}\right)\bigg\}_{k=1}^U \Bigg|b_t^{(k)}=\left|\frac{\sqrt{P_k^{\max}}h_{k,t}}{K_k(|w_{t-1}|+\eta)}\right|, \bm{\beta}_t^{(k)}(b_t^{(k)})=\left[\beta^{(k)}_{1,t},\dots, \beta^{(k)}_{U,t}\right],k = 1,\dots, U\right\},
\end{align}
where $\bm{\beta}_t^{(k)}$ is a function of $b_t^{(k)}$, in the form:
 \begin{align}\label{eq:betak0}
\beta^{(k)}_{i,t}=H\bigg(P_i^{\max}-\bigg|\frac{K_ib^{(k)}_{t}(|w_{t-1}|+\eta)}{h_{i,t}}\bigg|\bigg)
\end{align}
and $H(x)$ is the Heaviside step function, i.e., $H(x)=1$ for $x>0$, and $H(x)=0$ otherwise.
\end{theorem}

\begin{proof}
Please see Appendix \ref{Appendix B}.
\end{proof}

Thanks to \textbf{Theorem \ref{Theorem_SolutionSpace}}, we equivalently transform \textbf{P3} from a MIP into a discrete programming (DP) problem \textbf{P4} as follows
\begin{align}\label{P3eq}
\textbf{P4:}  \qquad  \min_{(b_t, \boldsymbol{\beta}_t) \in \mathcal{S}} R_t = R_t\left(b_t, \boldsymbol{\beta}_t\right)
\end{align}

According to \textbf{P4}, the objective $R_t$ can only take on $U$ possible values corresponding to the $U$ feasible values of $b_t$; meanwhile, given each $b_t$, the value of $\boldsymbol{\beta}_t$ is uniquely determined. Hence the minimum $R_t$ can be obtained via line search over the $U$ feasible points $(b_t, \boldsymbol{\beta}_t)$ in \eqref{space}. Note that the feasible points in \eqref{space} are determined by the channel gains, power limits and data sizes of the $U$ workers. Hence, the optimal transmission policy decided by \textbf{P4} reflects the tradeoff among workers in terms of their channel quality, available power resource, and data quality.


It is worth noting that the solution $b^*_{t}$ of \textbf{P4} may exceed the maximum value allowed at a worker, due to the approximation introduced in \textbf{Assumption 4}. To strictly comply to the power constraint, each worker needs to take the following bounding step when sending its local parameters: 

1) if $\bigg|\frac{K_ib^*_{t}w_{i,t}}{h_{i,t}}\bigg|^2 \leq P_i^{\max}$, then the $i$-th worker sends $\frac{K_ib^*_{t}w_{i,t}}{h_{i,t}}$ ;

2) otherwise, it sends $\sqrt{P_i^{\max}}\text{sgn}(w_{i,t})$, where $\mbox{sgn}(\cdot)$ is the sign function.

Putting together, we propose a jo\textbf{i}nt optimizatio\textbf{n} for \textbf{FL} \textbf{o}ver \textbf{t}he \textbf{a}ir (INFLOTA) as summarized in \textbf{Algorithm \ref{alg:policyforFL}}, which is a dynamic scheduling and power scaling policy. By using different $R_t$, our INFLOTA can be adjust to all the considered cases including the convex and non-convex, using either GD or SGD implementations.


\begin{algorithm}[!htb]
	\caption{The implementation of INFLOTA}
	\label{alg:policyforFL}
	\begin{algorithmic}[1]
\renewcommand{\algorithmicrequire}{\textbf{Given:}}
		\REQUIRE ~~\\
		System parameters $\{P_{i}^{\max}, K_i, \eta\}_{i=1}^U$;
\\
\STATE The PS initializes $\{\mathbf{w}_0, b_{1}^*, \bm{\beta}^*_{1}\}$ and broadcasts them to all the local workers.
    \FOR {$t=1:T$}
    \STATE \!\!\!\!\textbf{At the workers:}
        \STATE \textbf{Computation:} Iteratively update the local model via \eqref{eq:localupdate0} where $\mathbf{w}$ = $\mathbf{w}_{t-1}$ is from the PS; 
        \STATE \textbf{Communication:} Upon receiving $(b_{t}, \bm{\beta}_{t})$, send $\text{sgn}(w_{i,t})\min \left(\frac{K_ib_{t}|w_{i,t}|}{h_{i,t}}, \sqrt{P_i^{\texttt{Max}}}\right)$ to the PS, if $\beta_{i,t}=1, \forall i, d$;
    \STATE \!\!\!\!\textbf{At the PS:}
    \STATE Calculate the global model $\mathbf{w}_{t}$ via \eqref{eq:gt} from $\mathbf{y}_t$ that is aggregated from local workers; 
        \FOR {$d=1:D$}
        \STATE Calculate $\mathcal{S}$ from \eqref{space}, which yields $U$ feasible points $\{(b_{t+1}^{(k)}, \bm{\beta}_{t+1}^{(k})\}_{k=1}^U$;
        \STATE Solve \textbf{P4} in \eqref{P3eq} by using a line search over the $U$ feasible points to find the optimal $\{b_{t+1}^*, \bm{\beta}^*_{t+1}\}$ for given $d$ and $t$;
        \ENDFOR
        \STATE Send $w_{t}$ and $(b_{t+1}, \bm{\beta}_{t+1})$ (including all $D$ optimal $\{b_{t+1}^*, \bm{\beta}^*_{t+1}\}$) to all the workers;
    \ENDFOR
	\end{algorithmic}
\end{algorithm}

\begin{remark}
(\textbf{Optimality})
\textbf{P3} is equivalently reformulated as \textbf{P4}, which is solved by a search method with much reduced computational complexity thanks to the reduced search space.
Comparing \textbf{P3} to \textbf{P2}, the constraint \eqref{constraint:P3b} of \textbf{P3} is more restrictive than the constraint in \eqref{con:pmax} of \textbf{P2}. Since \textbf{P3} reduces the feasible domain of \textbf{P2}, the solution to \textbf{P3} cannot be superior to that of \textbf{P2}. Therefore, the optimal solution of \textbf{P3} is an upper bound of \textbf{P2}, i.e., $R_t$ calculated by solving \textbf{P3} is larger than the actual one.
\end{remark}

\begin{remark}
 (\textbf{Complexity}) \textbf{Algorithm \ref{alg:policyforFL}} provides a holistic solution for implementation of the overall FL at both the PS and workers sides. Its computational complexity is mainly determined by that of the optimization step in \textbf{P4}.
The complexity order of the optimization step is low at $\mathcal{O}(U)$, since the search space is reduced to $U$ points only via \textbf{P4}. 
\end{remark}


\begin{remark}\label{remark:implementation}
(\textbf{Implementation}) To implement the FL over the air in \textbf{Algorithm \ref{alg:policyforFL}}, the PS must know the CSI, the number of the data samples and the maximum transmit power of all local workers. This information can be obtained by the PS when local workers initially connect to it. Before the implementation of \textbf{Algorithm \ref{alg:policyforFL}}, the PS must first broadcast the global model information to the workers.
Noticeably, taking \textbf{P4} into the implementation of FL, some workers need to send $\text{sgn}(w_{i,t})\min \left(\frac{K_ib_{t}|w_{i,t}|}{h_{i,t}},  \sqrt{P_i^{\texttt{Max}}}\right)$ to meet the requirement on the maximum transmit power. Such a bounding method can be viewed as a quantization measure, which does not affect the convergence \cite{alistarh2017qsgd}.
\end{remark}

\section{Simulation Results And Analysis} \label{Sec:Numerical Results}

In the simulations, we evaluate the performance of the proposed INFLOTA for both linear regression and image classification tasks, which are based on a synthetic dataset and the MNIST dataset, respectively.

The considered network has $U=20$ workers, whose maximum power is set to be $P^{\max}_i=P^{\max}= 10$ mW for any $i \in [1, U]$. The receiver noise power at PS is set to be $\sigma^2=10^{-4}$ mW, i.e., $SNR=\frac{P^{\max}}{\sigma^2}=5$ dB. The wireless channel gain between the workers and the PS are generated from a Rayleigh fading model. Here, $h_{i,t}$ is generated from an exponential distribution with unit mean for different $i$ and $t$.

We use two baseline methods for comparison: a) an FL algorithm that assumes idealized wireless transmissions with error-free links to achieve perfect aggregation, and b) an FL algorithm that randomly determines the power scalar and user selection. They are named as \emph{Perfect aggregation} and \emph{Random policy}, respectively.  In \emph{Random policy}, the probability of each worker being selected is 50\% and the power scalar is generated from an exponential distribution with unit mean.

\subsection{Linear regression experiments}
In linear regression experiments, the synthetic data used to train the FL algorithm is generated randomly from $[0,1]$. The input $x$ and the output $y$ follow the function $y= -2x + 1 + n\times 0.4$ where $n$ follows a Gaussian distribution $\mathcal{N}(0,1)$. The FL algorithm is used to model the relationship between $x$ and $y$.

\begin{figure}[tb]
  \centering
  \includegraphics[scale=0.45]{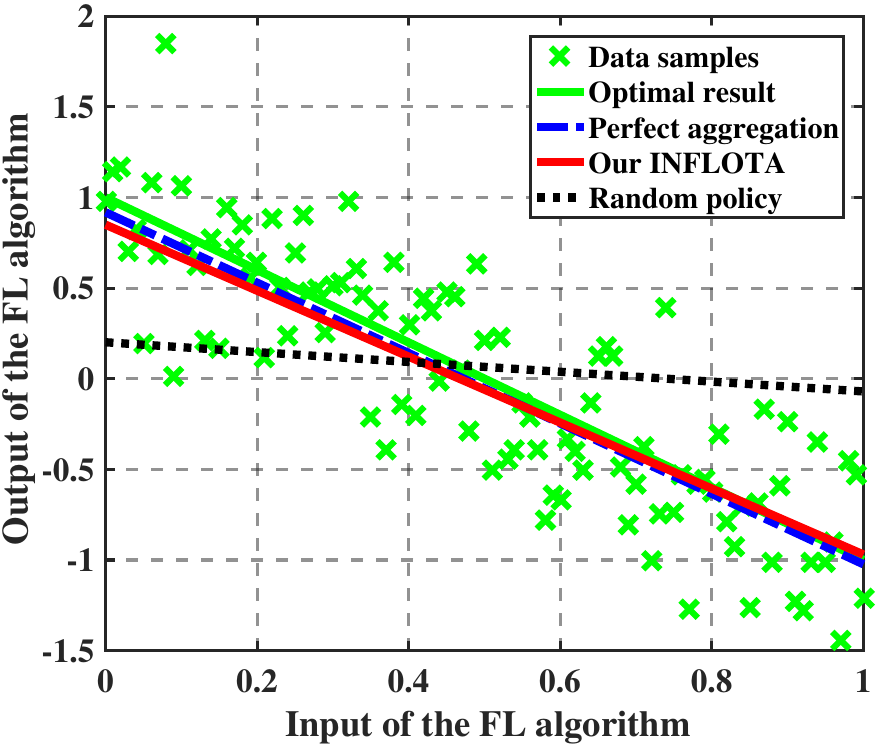}
  \caption{An example of implementing FL for linear regression.}\label{linearregression}
\end{figure}
Since linear regression only involves two parameters, we train a simple two-layer neural network, with one neuron in each layer, without activation functions, which is the convex case. The loss function is the MSE of the model prediction $\hat{y}$ and the labeled true output $y$. The learning rate is set to 0.01. 

Fig. \ref{linearregression} shows an example of using FL for linear regression. The optimal result of a linear regression is $y=-2x+1$, because the original data generation function is $y=-2x+1+0.4n$.
In Fig. \ref{linearregression}, we can see that the most accurate approximation is achieved by \emph{Perfect aggregation}, which is the ideal case without considering the influence of wireless communication.  \emph{Random policy} considers the influence of wireless communication but without any optimization. Thus, its performance is worst. Our proposed INFLOTA performs closely to the ideal case, which jointly considers the learning and the influence of wireless communication. This is because that our proposed INFLOTA can optimize worker selection and power control so as to reduce the effect of wireless transmission errors on FL.
\begin{figure}[tb]
  \centering
  \includegraphics[scale=0.45]{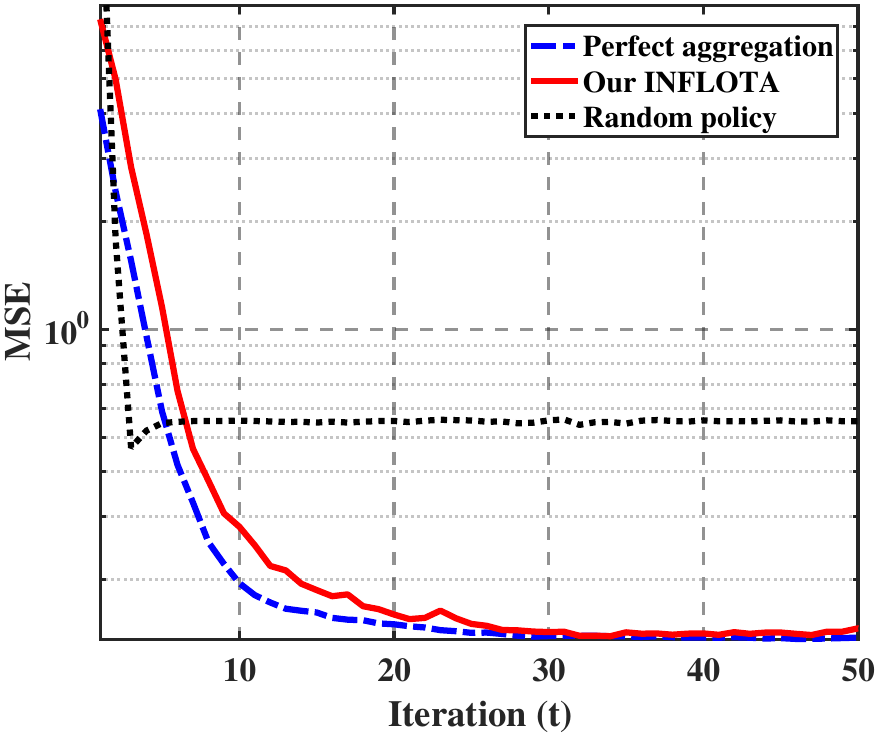}
  \caption{MSE as the number of iteration varies.}\label{iteration}
\end{figure}

In Fig. \ref{iteration}, we show how wireless transmission affects the convergence behavior of the global FL model training in terms the value of the loss function and the global FL model remains unchanged which shows that the global FL model converges. As we can see, as the number of iterations increases, the MSE values of all the considered learning algorithms decrease at different rates, and eventually flatten out to reach their steady state. All schemes converge, but to different steady state values. This behavior corroborates  the results in \textbf{Lemma \ref{Lemma1}} and \textbf{Proposition \ref{Propositon1}} that the channel noise does not affect the convergence of the FL algorithm but it affects the value that the FL algorithm converges to.
\begin{figure}[tb]
  \centering
  \includegraphics[scale=0.45]{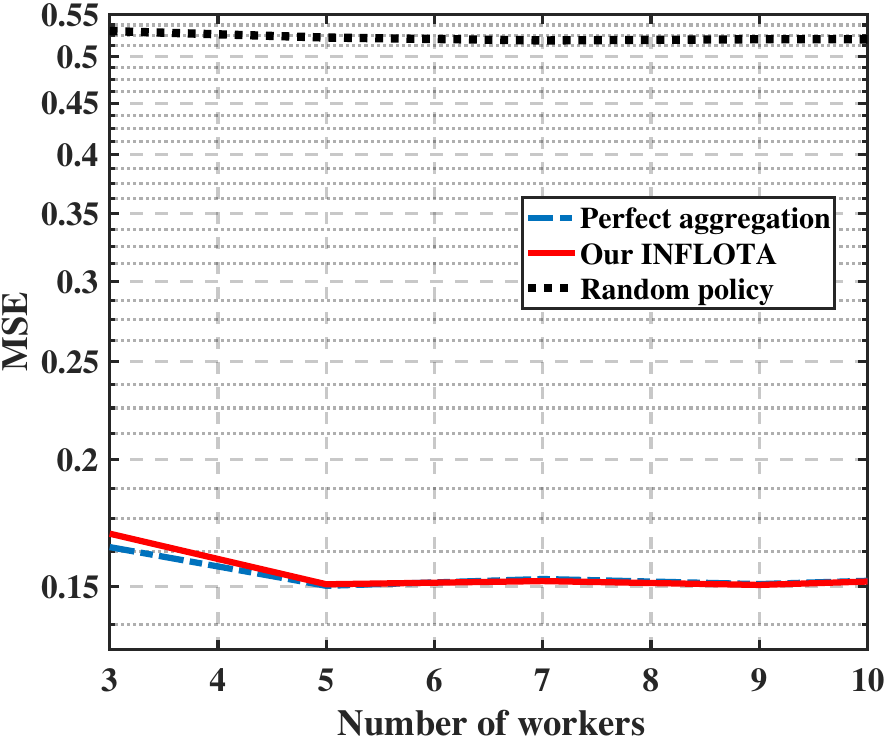}
  \caption{MSE as the number of workers varies.}\label{nubofworker}
\end{figure}

Fig. \ref{nubofworker} shows how MSE varies with the total number of workers $U$. In general, the MSE performance of all considered FL schemes decreases as $U$ increases. This is due to the fact that an increase in the number of workers leads to an increased volume of data available for the FL training and hence improved accuracy of the estimated model parameters. Moreover, as the number of workers increases, the effect of wireless transmission on the global FL model accuracy starts to diminish. This is because the data samples may be already enough for accurate training when $U$ exceeds a certain level.  

\begin{figure}[tb]
  \centering
  \includegraphics[scale=0.45]{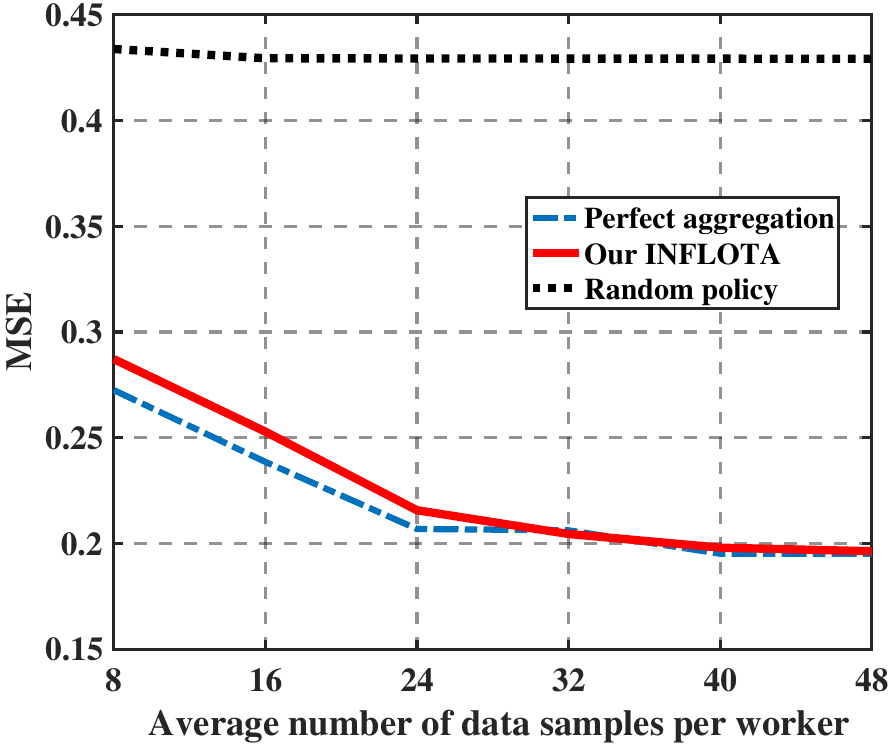}
  \caption{MSE as the number of data samples per worker varies.}\label{samples}
\end{figure}

In Fig. \ref{samples}, we present how the MSE changes with the average number of samples per worker $\bar{K} = K/U$. The number of data samples per worker fluctuates around the average number, i.e., we set $K_i=\text{round}(\text{uniform}[\bar{K}-5, \bar{K}+5])$. As $\bar{K}$ increases, all of the considered learning algorithms have more data samples available for training, and hence the MSE of all of considered FL algorithms decrease in Fig. \ref{samples}. As the average data samples per worker continues to increase, the MSE improvement slows down and eventually saturates.
This is because that as the data samples per worker continues to increase, the data samples  are enough for training the FL model. 

\begin{figure}[tb]
  \centering
  \includegraphics[scale=0.45]{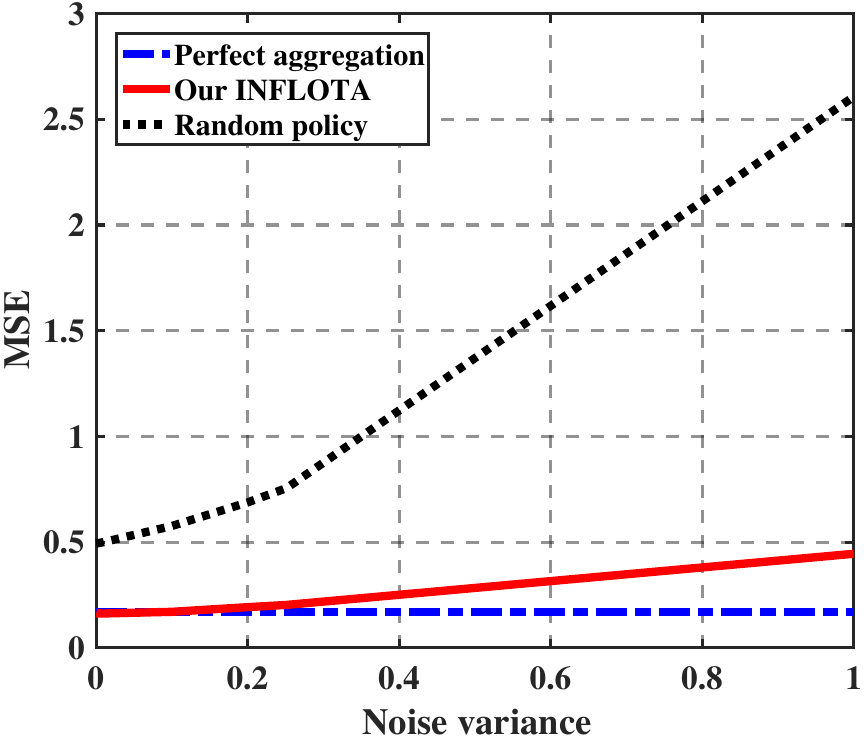}
    \caption{MSE as the noise variance varies.}\label{noise}
\end{figure}

Fig. \ref{noise} presents how the AWGN received by the PS affects the MSE. We can see that as the noise variance increases, the MSE values of all of considered FL algorithms increase, except for \emph{Perfect aggregation}. When the noise variance is small (e.g., less than $10^{-1}$), it has little effect on the performance of FL algorithms. 

\subsection{Evaluation on the MNIST dataset}

In order to evaluate the performance of our proposed INFLOTA in realistic application scenarios with real data, we train a multilayer perceptron (MLP) on the MNIST dataset\footnote{http://yann.lecun.com/exdb/mnist/} with a 784-neuron input layer, a 64-neuron hidden layer, and a 10-neuron softmax output layer, which is a non-convex case. We adopt cross entropy as the loss function, and rectified linear unit (ReLU) as the activation function. The total number of parameters in the MLP is 50890. The learning rate $\alpha$ is set as 0.1. In MNIST dataset, there are 60000 training samples and 10000 test samples. We randomly take out $500-1000$ training samples and distribute them to 20 local workers as their local data. Then the three trained FL are tested with 10000 test samples. We provide the results of cross entropy and test accuracy versus the iteration index $t$ in Fig. \ref{crossentropy} and Fig. \ref{TESTaccuracy}, respectively. Since the MNIST dataset is designed for handwritten digit identification, the test accuracy presents the identification accuracy. As we can see, our proposed INFLOTA outperforms \emph{Random policy}, and achieves comparable performance as \emph{Perfect aggregation}.

\begin{figure}[tb]
  \centering
  \includegraphics[scale=0.45]{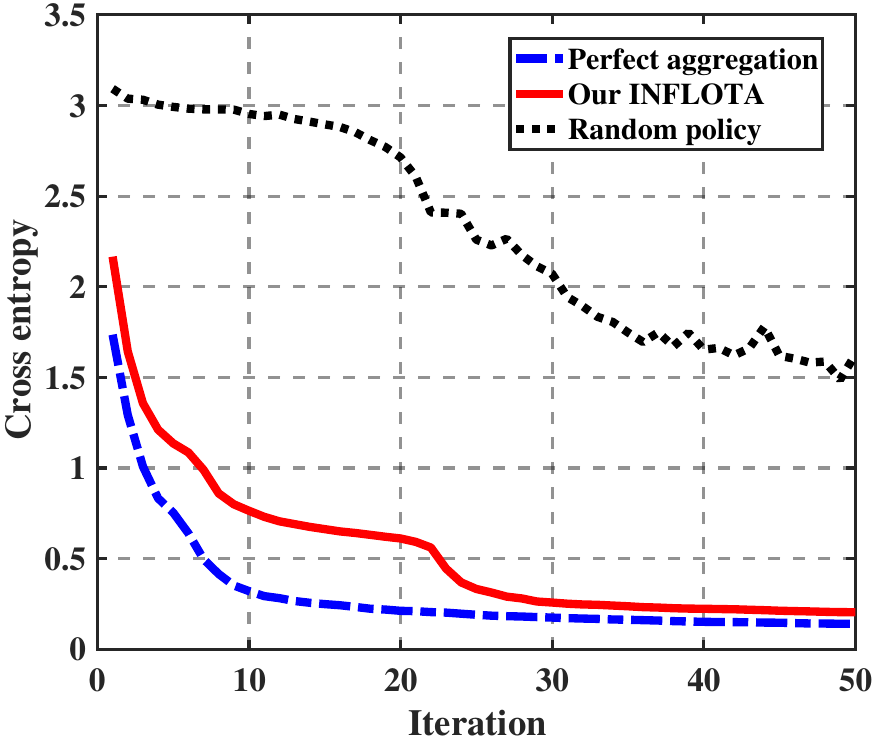}
    \caption{Cross entropy as the number of iteration varies.}\label{crossentropy}
\end{figure}

\begin{figure}[tb]
  \centering
  \includegraphics[scale=0.45]{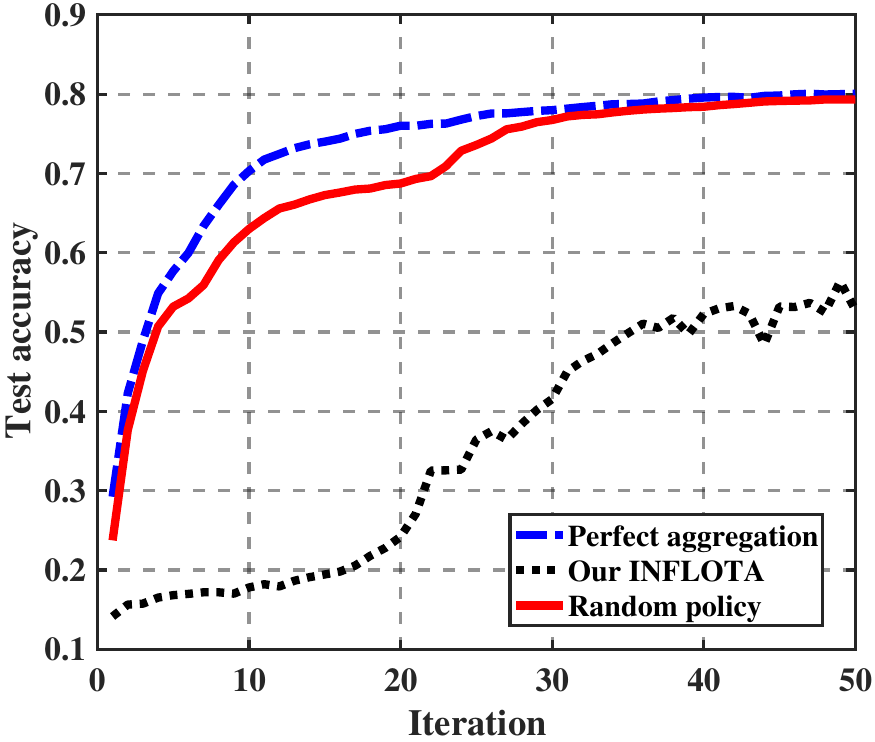}
    \caption{Test accuracy as the number of iteration varies.}\label{TESTaccuracy}
\end{figure}

\section{Conclusion}\label{Sec:Conclusion}
In this paper, we have studied the joint optimization of communications and
FL over the air with analog aggregation, in which both worker selection and transmit power control are considered under the constraints of limited communication resources. Under the convex and non-convex cases with either the GD or SGD implementations, we respectively derive closed-form expressions for the expected convergence rate of the FL algorithm, which can quantify the impact of resource-constrained wireless communications on FL under the analog aggregation paradigm. Through analyzing the expected convergence rate, we have proposed a joint optimization scheme of worker selection and power control, which can mitigate the impact of wireless communications on the convergence and performance of the FL algorithm. More significantly, our joint optimization scheme is applicable for both the convex and non-convex cases, using either GD or SGD implementations. Simulation results show that the proposed optimization scheme is effective in mitigating the impact of wireless communications on FL.

\section*{Acknowledgments}
We are very grateful to all reviewers who have helped improve the quality of this paper. This work was partly supported by the National Natural Science Foundation of China (Grant Nos. 61871023 and 61931001), Beijing Natural Science Foundation (Grant No. 4202054), and the National Science Foundation of the US (Grant Nos. 1741338 and 1939553).


\begin{appendices}
\section{Proof of \textbf{Theorem \ref{Theorem1}}}\label{Appendix A}
\textbf{Theorem \ref{Theorem1}} considers the full GD method for convex problems. Following the proof for the gradient methods with noise in \cite{friedlander2012hybrid}, we first present the inequality implied by the \textbf{Assumption 1}, as follows
\begin{align}\label{Taylor}
F(\mathbf{w}_{t})
\leq F(\mathbf{w}_{t-1})+(\mathbf{w}_{t}-\mathbf{w}_{t-1})^T\nabla F(\mathbf{w}_{t-1})+\frac{L}{2}\|\mathbf{w}_{t}-\mathbf{w}_{t-1}\|^2.
\end{align}

Employing a standard full GD method, the $i$-th worker updates its local FL model parameter $\mathbf{w}_{i,t}$ at the $t$-th iteration by
\begin{equation}\label{eq:localupdate}
  \mathbf{w}_{i,t}=\mathbf{w}_{t-1}-\frac{\alpha}{K_i}\sum_{k=1}^{K_i}\nabla f(\mathbf{w}_{t-1},\mathbf{x}_{i,k},\mathbf{y}_{i,k}), \quad i=1,2,...,U.
\end{equation}

Substituting \eqref{eq:localupdate} to \eqref{eq:gt}, we have
\begin{align}\label{13}
\mathbf{w}_{t}=& \mathbf{w}_{t-1}+\left(\sum_{i=1}^U K_i\bm{\beta}_{i,t}\odot\bm{b}_{t}\right)^{\odot-1}\odot\mathbf{z}_{t}\nonumber\\
  &-\alpha\left(\sum_{i=1}^U K_i\bm{\beta}_{i,t}\right)^{\odot-1}\odot \sum_{i=1}^U\sum_{k=1}^{K_i}\bm{\beta}_{i,t}\odot\nabla f(\mathbf{w}_{t-1};\mathbf{x}_{i,k},y_{i,k})\nonumber\\
=&\mathbf{w}_{t-1}-\alpha(\nabla F(\mathbf{w}_{t-1})-\mathbf{o}),
\end{align}
where
\begin{align}\label{14}
\mathbf{o}=&\nabla F(\mathbf{w}_{t-1})+\left(\alpha\sum_{i=1}^U K_i\bm{\beta}_{i,t}\odot\bm{b}_{t}\right)^{\odot-1}\odot\mathbf{z}_{t}
\nonumber\\&-\left(\sum_{i=1}^U K_i\bm{\beta}_{i,t}\right)^{\odot-1}\odot \sum_{i=1}^U\sum_{k=1}^{K_i}\bm{\beta}_{i,t}\odot\nabla f(\mathbf{w}_{t-1};\mathbf{x}_{i,k},y_{i,k}).
\end{align}

Given the learning rate $\alpha=\frac{1}{L}$ (a special setting for simple expression without loss of generality), the expected optimization function of $\mathbb{E}[F(\mathbf{w}_{t})]$ can be expressed as
\begin{align}\label{eq:Egt}
\mathbb{E}[F(\mathbf{w}_{t})]\leq & \mathbb{E}\bigg[F(\mathbf{w}_{t-1})-\alpha(\nabla F(\mathbf{w}_{t-1})-\mathbf{o})^T\nabla F(\mathbf{w}_{t-1})
+\frac{L\alpha^2}{2}\|\nabla F(\mathbf{w}_{t-1})-\mathbf{o}\|^2\bigg]\nonumber\\
\overset{(a)}{=}&\mathbb{E}[F(\mathbf{w}_{t-1})]-\frac{1}{2L}\|\nabla F(\mathbf{w}_{t-1})\|^2+\frac{1}{2L}\mathbb{E}[\| \mathbf{o}\|^2],
\end{align}
where the step (a) is derived from the fact that
\begin{align}
\frac{L\alpha^2}{2}\|\nabla F(\mathbf{w}_{t-1})-\mathbf{o}\|^2&=\frac{1}{2L}\|\nabla F(\mathbf{w}_{t-1})\|^2-\frac{1}{L}\mathbf{o}^T\nabla F(\mathbf{w}_{t-1})+\frac{1}{2L}\| \mathbf{o}\|^2.
\end{align}

$\mathbb{E}[\| \mathbf{o}\|^2]$ can be derived as follows
\begin{align}\label{eq:Eo}
\mathbb{E}[\| \mathbf{o}\|^2]=&\mathbb{E}\Bigg[\bigg\|\nabla F(\mathbf{w}_{t-1})+\left(\alpha\sum_{i=1}^U K_i\bm{\beta}_{i,t}\odot\bm{b}_{t}\right)^{\odot-1}\odot\mathbf{z}_{t}
\nonumber\\
&-\left(\sum_{i=1}^U K_i\bm{\beta}_{i,t}\right)^{\odot-1}\odot \sum_{i=1}^U\sum_{k=1}^{K_i}\bm{\beta}_{i,t}\odot\nabla f(\mathbf{w}_{t-1};\mathbf{x}_{i,k},\mathbf{y}_{i,k}) \bigg\|^2\Bigg]\nonumber\\ \nonumber
=&\mathbb{E}\Bigg[\bigg\|\frac{\sum_{i=1}^U\sum_{k=1}^{K_i}\nabla f(\mathbf{w}_{t-1};\mathbf{x}_{i,k},\mathbf{y}_{i,k})}{K}
+\left(\alpha\sum_{i=1}^U K_i\bm{\beta}_{i,t}\odot\bm{b}_{t}\right)^{\odot-1}\odot\mathbf{z}_{t}\\ \nonumber
&-\left(\sum_{i=1}^U K_i\bm{\beta}_{i,t}\right)^{\odot-1}\odot \sum_{i=1}^U\sum_{k=1}^{K_i}\bm{\beta}_{i,t}\odot\nabla f(\mathbf{w}_{t-1};\mathbf{x}_{i,k},\mathbf{y}_{i,k})  \bigg\|^2\Bigg] \\
=&\mathbb{E}\Bigg[\bigg\|\sum_{i=1}^U\left(\frac{\mathbf{1}}{K}-\bm{\beta}_{i,t}\odot\left(\sum_{i=1}^U K_i\bm{\beta}_{i,t}\right)^{\odot-1}\right)\odot\sum_{k=1}^{K_i}\nabla f(\mathbf{w}_{t-1};\mathbf{x}_{i,k},\mathbf{y}_{i,k})\nonumber\\
&
+\left(\alpha\sum_{i=1}^U K_i\bm{\beta}_{i,t}\odot\bm{b}_{t}\right)^{\odot-1}\odot\mathbf{z}_{t} \bigg\|^2\Bigg],
\end{align}
where $\mathbf{1}$ is the all-$1$ vector of length $D$. The dimension of $\mathbf{1}$ is the same length as that of $\bm{\beta}_{i,t}$.

Employing the triangle inequality of norms $\| \mathbf{X}+\mathbf{Y}\| \leq \| \mathbf{X}\|+\| \mathbf{Y}\|$, the submultiplicative property of norms $\| \mathbf{X}\mathbf{Y}\| \leq \| \mathbf{X}\|\| \mathbf{Y}\|$, and the Jensen’s inequality, \eqref{eq:Eo} can be further derived as follows
\begin{align}\label{eq:Eo1}
\mathbb{E}[\| \mathbf{o}\|^2]
\leq &\mathbb{E}\Bigg[\Bigg\|\sum_{i=1}^U\left(\frac{\mathbf{1}}{K}-\bm{\beta}_{i,t}\odot\left(\sum_{i=1}^U K_i\bm{\beta}_{i,t}\right)^{\odot-1}\right)\odot\sum_{k=1}^{K_i}\nabla f(\mathbf{w}_{t-1};\mathbf{x}_{i,k},\mathbf{y}_{i,k})\Bigg\|^2\Bigg]
\nonumber\\ \nonumber &+\mathbb{E}\Bigg[\Bigg \|\left(\alpha\sum_{i=1}^U K_i\bm{\beta}_{i,t}\odot\bm{b}_{t}\right)^{\odot-1}\odot\mathbf{z}_{t}\Bigg\|^2\Bigg] \nonumber\\ \nonumber
\leq &\mathbb{E}\Bigg[K \sum_{i=1}^U\Bigg\|\frac{\mathbf{1}}{K}-\bm{\beta}_{i,t}\odot\left(\sum_{i=1}^U K_i\bm{\beta}_{i,t}\right)^{\odot-1}\Bigg\|^2\sum_{k=1}^{K_i}\|\nabla f(\mathbf{w}_{t-1};\mathbf{x}_{i,k},\mathbf{y}_{i,k})\|^2\Bigg] \nonumber\\ \nonumber &+\mathbb{E}\Bigg[\left\|\left(\alpha\sum_{i=1}^U K_i\bm{\beta}_{i,t}\odot\bm{b}_{t}\right)^{\odot-1}\odot\mathbf{z}_{t}\right\|^2\Bigg]\\
\leq &K \sum_{i=1}^U\Bigg\|\frac{\mathbf{1}}{K}-\bm{\beta}_{i,t}\odot\left(\sum_{i=1}^U K_i\bm{\beta}_{i,t}\right)^{\odot-1}\Bigg\|^2\sum_{k=1}^{K_i}\|\nabla f(\mathbf{w}_{t-1};\mathbf{x}_{i,k},\mathbf{y}_{i,k})\|^2
\nonumber\\ &+\left\|\left(\sum_{i=1}^U K_i\bm{\beta}_{i,t}\odot\bm{b}_{t}\right)^{\odot-1}\right\|^2\sigma^2L^2.
\end{align}

Applying \eqref{eq:bound} in \textbf{Assumption 3} to \eqref{eq:Eo1} leads to

\begin{align}\label{eq:Eo2}
\mathbb{E}[\| \mathbf{o}\|^2]
\leq&K\sum_{i=1}^U\Bigg\|\frac{\mathbf{1}}{K}-\bm{\beta}_{i,t}\odot\left(\sum_{i=1}^U K_i\bm{\beta}_{i,t}\right)^{\odot-1}\Bigg\|^2K_i(\rho_1+\rho_2\|\nabla F(\mathbf{w}_{t-1})\|^2)\nonumber\\
&+\left\|\left(\sum_{i=1}^U K_i\bm{\beta}_{i,t}\odot\bm{b}_{t}\right)^{\odot-1}\right\|^2\sigma^2L^2\nonumber\\
=&K\sum_{i=1}^U\sum_{d=1}^D\Bigg(\frac{1}{K}-\frac{\beta^d_{i,t}}{\sum_{i=1}^U K_i\beta^d_{i,t}}\Bigg)^2K_i(\rho_1+\rho_2\|\nabla F(\mathbf{w}_{t-1})\|^2)\nonumber\\
&+\left\|(\sum_{i=1}^U K_i\bm{\beta}_{i,t}\odot\bm{b}_{t})^{\odot-1}\right\|^2\sigma^2L^2
\nonumber\\
=&K\sum_{i=1}^U\sum_{d=1}^D\Bigg(\frac{1}{K^2}-\frac{2}{K}\frac{\beta^d_{i,t}}{\sum_{i=1}^U K_i\beta^d_{i,t}}+\frac{(\beta^d_{i,t})^2}{(\sum_{i=1}^U K_i\beta^d_{i,t})^2}\Bigg)K_i(\rho_1+\rho_2\|\nabla F(\mathbf{w}_{t-1})\|^2)\nonumber\\
&+\left\|(\sum_{i=1}^U K_i\bm{\beta}_{i,t}\odot\bm{b}_{t})^{\odot-1}\right\|^2\sigma^2L^2
\nonumber\\
=&\sum_{d=1}^D\Bigg(\frac{K}{\sum_{i=1}^U K_i\beta^d_{i,t}}-1\Bigg)(\rho_1+\rho_2\|\nabla F(\mathbf{w}_{t-1})\|^2)+\left\|(\sum_{i=1}^U K_i\bm{\beta}_{i,t}\odot\bm{b}_{t})^{\odot-1}\right\|^2\sigma^2L^2.
\end{align}

Substituting \eqref{eq:Eo2} to \eqref{eq:Egt}, we have:
\begin{align}\label{eq:Egt1}
\mathbb{E}[F(\mathbf{w}_{t})]
\leq&\frac{1}{2L}\Bigg(\sum_{d=1}^D\Bigg(\frac{K}{\sum_{i=1}^U K_i\beta^d_{i,t}}-1\Bigg)(\rho_1+\rho_2\|\nabla F(\mathbf{w}_{t-1})\|^2)\nonumber\\
&+\left\|(\sum_{i=1}^U K_i\bm{\beta}_{i,t}\odot\bm{b}_{t})^{\odot-1}\right\|^2\sigma^2L^2\Bigg) +\mathbb{E}[F(\mathbf{w}_{t-1})]-\frac{1}{2L}\|\nabla F(\mathbf{w}_{t-1})\|^2.
\end{align}

Subtract $\mathbb{E}[F(\mathbf{w}^*)]$ from both sides of \eqref{eq:Egt1}, we have:
\begin{align}\label{eq:EgtSubtract}
&\mathbb{E}[F(\mathbf{w}_{t})-F(\mathbf{w}^*)]
\leq\mathbb{E}[F(\mathbf{w}_{t-1})-F(\mathbf{w}^*)]-\frac{1}{2L}\|\nabla F(\mathbf{w}_{t-1})\|^2\nonumber\\
&+\frac{1}{2L}\Bigg(\sum_{d=1}^D\Bigg(\frac{K}{\sum_{i=1}^U K_i\beta^d_{i,t}}-1\Bigg)(\rho_1+\rho_2\|\nabla F(\mathbf{w}_{t-1})\|^2)+\left\|(\sum_{i=1}^U K_i\bm{\beta}_{i,t}\odot\bm{b}_{t})^{\odot-1}\right\|^2\sigma^2L^2\Bigg).
\end{align}

To minimize both sides of \eqref{eq:stronglyconvex}, we have
\begin{align}\label{19}
\min_{\mathbf{w}_{t}}F(\mathbf{w}_{t})\geq \min_{\mathbf{w}_{t}}(& F(\mathbf{w}_{t-1})+(\mathbf{w}_{t}-\mathbf{w}_{t-1})^T\nabla F(\mathbf{w}_{t-1})+\frac{\mu}{2}\|\mathbf{w}_{t}-\mathbf{w}_{t-1}\|^2).
\end{align}

The minimization of the left-hand side is achieved by $\mathbf{w}_{t}=\mathbf{w}^*$, while the minimization of the right-hand side is achieved by $\mathbf{w}_{t}=\mathbf{w}_{t-1}-\frac{1}{\mu}\nabla F(\mathbf{w}_{t-1})$.
Thus, we have
\begin{align}\label{20}
F(\mathbf{w}^*)\geq F(\mathbf{w}_{t-1})-\frac{1}{2\mu}\|\nabla F(\mathbf{w}_{t-1})\|^2 .
\end{align}
Then
\begin{align}\label{eq:minimizationlr}
\|\nabla F(\mathbf{w}_{t-1})\|^2\geq 2\mu(F(\mathbf{w}_{t-1})-F(\mathbf{w}^*)).
\end{align}

Substituting \eqref{eq:minimizationlr} to \eqref{eq:EgtSubtract}, we get
\begin{align}\label{eq:EgtStronglycon}
&\mathbb{E}[F(\mathbf{w}_{t})-F(\mathbf{w}^*)]
\leq(1-\frac{\mu}{L})\mathbb{E}[F(\mathbf{w}_{t-1})-F(\mathbf{w}^*)]\nonumber\\
&+\frac{1}{2L}\Bigg(\sum_{d=1}^D\Bigg(\frac{K}{\sum_{i=1}^U K_i\beta^d_{i,t}}-1\Bigg)(\rho_1+\rho_2\|\nabla F(\mathbf{w}_{t-1})\|^2)+\left\|(\sum_{i=1}^U K_i\bm{\beta}_{i,t}\odot\bm{b}_{t})^{\odot-1}\right\|^2\sigma^2L^2\Bigg)
\nonumber\\
&=(1-\frac{\mu}{L})\mathbb{E}[F(\mathbf{w}_{t-1})-F(\mathbf{w}^*)]+\frac{\rho_2}{2L}\Bigg(\sum_{d=1}^D\Bigg(\frac{K}{\sum_{i=1}^U K_i\beta^d_{i,t}}-1\Bigg)\Bigg)\|\nabla F(\mathbf{w}_{t-1})\|^2\nonumber\\
&+\frac{\rho_1}{2L}\sum_{d=1}^D\Bigg(\frac{K}{\sum_{i=1}^U K_i\beta^d_{i,t}}-1\Bigg)+\left\|(\sum_{i=1}^U K_i\bm{\beta}_{i,t}\odot\bm{b}_{t})^{\odot-1}\right\|^2\frac{L\sigma^2}{2}.
\end{align}

Next, in the same way that \eqref{eq:minimizationlr} is derived, to minimize both sides of \eqref{Taylor}, we have
\begin{align}\label{eq:minimizationSmooth}
\|\nabla F(\mathbf{w}_{t-1})\|^2\leq 2L(F(\mathbf{w}_{t-1})-F(\mathbf{w}^*)).
\end{align}

Substituting \eqref{eq:minimizationSmooth} to \eqref{eq:EgtStronglycon}, we get
\begin{align}\label{eq:Egtsmoothly}
\mathbb{E}[F(\mathbf{w}_{t})-F(\mathbf{w}^*)]
\leq&\Bigg(1-\frac{\mu}{L}+\rho_2\sum_{d=1}^D\Bigg(\frac{K}{\sum_{i=1}^U K_i\beta^d_{i,t}}-1\Bigg)\Bigg)\mathbb{E}[F(\mathbf{w}_{t-1})-F(\mathbf{w}^*)] \nonumber\\
&+\frac{\rho_1}{2L}\sum_{d=1}^D\Bigg(\frac{K}{\sum_{i=1}^U K_i\beta^d_{i,t}}-1\Bigg)+\left\|(\sum_{i=1}^U K_i\bm{\beta}_{i,t}\odot\bm{b}_{t})^{\odot-1}\right\|^2\frac{L\sigma^2}{2}\nonumber\\
=&B_{t}+A_{t}\mathbb{E}[F(\mathbf{w}_{t-1})-F(\mathbf{w}^*)],
\end{align}
where
\begin{align}\label{eq:At}
A_{t}=1-\frac{\mu}{L}+\rho_2\sum_{d=1}^D\Bigg(\frac{K}{\sum_{i=1}^U K_i\beta^d_{i,t}}-1\Bigg),
\end{align}
\begin{align}\label{eq:Bt}
B_{t}&=\frac{\rho_1}{2L}\sum_{d=1}^D\Bigg(\frac{K}{\sum_{i=1}^U K_i\beta^d_{i,t}}-1\Bigg)+\left\|(\sum_{i=1}^U K_i\bm{\beta}_{i,t}\odot\bm{b}_{t})^{\odot-1}\right\|^2\frac{L\sigma^2}{2}.
\end{align}

The proof is completed.

\section{Proof of \textbf{Theorem \ref{Theoremnonconvex}}}\label{Appendix B}
\textbf{Theorem \ref{Theoremnonconvex}} considers the full GD method for non-convex problems. The proof of \textbf{Theorem \ref{Theoremnonconvex}} follows that of \textbf{Theorem \ref{Theorem1}} until \eqref{eq:Egt1}. From \eqref{eq:Egt1}, we have
\begin{align}\label{eq:Egt1nonc}
\mathbb{E}[F(\mathbf{w}_{t})]
\leq&\mathbb{E}[F(\mathbf{w}_{t-1})]-\frac{1}{2L}\Bigg(1-\rho_2\sum_{d=1}^D\Bigg(\frac{K}{\sum_{i=1}^U K_i\beta^d_{i,t}}-1\Bigg)\Bigg)\|\nabla F(\mathbf{w}_{t-1})\|^2\nonumber\\
&+\frac{\rho_1}{2L}\sum_{d=1}^D\Bigg(\frac{K}{\sum_{i=1}^U K_i\beta^d_{i,t}}-1\Bigg)+\left\|(\sum_{i=1}^U K_i\bm{\beta}_{i,t}\odot\bm{b}_{t})^{\odot-1}\right\|^2\frac{L\sigma^2}{2}
\nonumber\\
&=\mathbb{E}[F(\mathbf{w}_{t-1})]-\frac{2-A_t-\frac{\mu}{L}}{2L}\|\nabla F(\mathbf{w}_{t-1})\|^2+B_t.
\end{align}

Summing up the above inequality from $t = 1$ to $t = T$, we get
\begin{align}\label{eq:Egt1nonc1}
\mathbb{E}[F(\mathbf{w}_{t})]-\mathbb{E}[F(\mathbf{w}_{0})]
\leq-\sum_{t=1}^T\frac{2-A_t-\frac{\mu}{L}}{2L}\|\nabla F(\mathbf{w}_{t-1})\|^2+\sum_{t=1}^TB_t,
\end{align}
which leads to
\begin{align}\label{eq:Egt1nonc2}
\sum_{t=1}^T\frac{2-A_t-\frac{\mu}{L}}{2L}\|\nabla F(\mathbf{w}_{t-1})\|^2 &\leq\mathbb{E}[F(\mathbf{w}_{0})]-\mathbb{E}[F(\mathbf{w}_{t})]+\sum_{t=1}^TB_t
\nonumber\\
&\leq\mathbb{E}[F(\mathbf{w}_{0})]-\mathbb{E}[F(\mathbf{w}^*)]+\sum_{t=1}^TB_t
.
\end{align}

Recalling \textbf{Proposition \ref{Propositon1}}, we have
\begin{align}\label{eq:Egt1noncAteq}
0\leq \frac{1-\rho_2D(\frac{K}{K_{min}}-1)}{2L}\leq \frac{2-A_t-\frac{\mu}{L}}{2L}\leq \frac{1}{2L}, \quad\forall t.
\end{align}

Substituting \eqref{eq:Egt1noncAteq} to \eqref{eq:Egt1nonc2}, we get
\begin{align}\label{eq:Egt1nonc3}
\frac{1}{T}\sum_{t=1}^T\frac{1-\rho_2D(\frac{K}{K_{min}}-1)}{2L}\|\nabla F(\mathbf{w}_{t-1})\|^2&\leq\frac{1}{T}\sum_{t=1}^T\frac{2-A_t-\frac{\mu}{L}}{2L}\|\nabla F(\mathbf{w}_{t-1})\|^2
\nonumber\\
&\leq\frac{1}{T}(\mathbb{E}[F(\mathbf{w}_{0})]-\mathbb{E}[F(\mathbf{w}^*)])+\frac{1}{T}\sum_{t=1}^TB_t
.
\end{align}

As a result, we have the conclusion in \textbf{Theorem \ref{Theoremnonconvex}},
\begin{align}\label{eq:Egt1noncresult}
\frac{1}{T}\sum_{t=1}^T\|\nabla F(\mathbf{w}_{t-1})\|^2\leq&\frac{2L }{T(1-\rho_2D(\frac{K}{K_{min}}-1))}\mathbb{E}[F(\mathbf{w}_{0})]-\mathbb{E}[F(\mathbf{w}^*)]\nonumber\\&+\frac{2L\sum_{t=1}^TB_t }{T(1-\rho_2D(\frac{K}{K_{min}}-1))}
.
\end{align}
\section{Proof of \textbf{Theorem \ref{SGDtheorem}}}\label{Appendix C}
Exploiting the SGD method, the local parameter of the $i$-th worker is updated at the $t$-th iteration by
\begin{align}\label{eq:localupdateSGD}
  \mathbf{w}_{i,t}=\mathbf{w}_{t-1}-\alpha \mathbb{E}_{\mathcal{D}_i}\left[\frac{\sum_{k=1}^{K_b}\nabla f(\mathbf{w}_{t-1},\mathbf{x}_{i,k},\mathbf{y}_{i,k})}{K_b}\right], \quad i=1,2,...,U,
\end{align}
where $\mathbb{E}_{\mathcal{D}_i}[\cdot]$ is the expectation, which represents that the $i$-th worker randomly chooses $K_b$ samples from its local dataset $\mathcal{D}_i$ to compute the local gradient.

Substituting \eqref{eq:localupdateSGD} to \eqref{eq:gt}, we reach a averaged gradient estimate as
\begin{align}\label{eq:SGDestimation}
\mathbf{w}_{t}=& \mathbf{w}_{t-1}+\left(\sum_{i=1}^U K_b\bm{\beta}_{i,t}\odot\bm{b}_{t}\right)^{\odot-1}\odot\mathbf{z}_{t}
  \nonumber\\&-\alpha\left(\sum_{i=1}^U K_b\bm{\beta}_{i,t}\right)^{\odot-1}\odot \sum_{i=1}^U\left(K_b\bm{\beta}_{i,t}\odot\mathbb{E}_{\mathcal{D}_i}\left[\frac{\sum_{k=1}^{K_b}\nabla f(\mathbf{w}_{t-1},\mathbf{x}_{i,k},\mathbf{y}_{i,k})}{K_b}\right]\right)\nonumber\\
=&\mathbf{w}_{t-1}-\alpha(\nabla F(\mathbf{w}_{t-1})-\mathbf{o}),
\end{align}
where
\begin{align}\label{14}
\mathbf{o}=&\nabla F(\mathbf{w}_{t-1})+\left(\alpha\sum_{i=1}^U K_b\bm{\beta}_{i,t}\odot\bm{b}_{t}\right)^{\odot-1}\odot\mathbf{z}_{t}
\nonumber\\&-\left(\sum_{i=1}^U K_b\bm{\beta}_{i,t}\right)^{\odot-1}\odot \sum_{i=1}^U\left(\bm{\beta}_{i,t}\odot\mathbb{E}_{\mathcal{D}_i}\left[\sum_{k=1}^{K_b}\nabla f(\mathbf{w}_{t-1},\mathbf{x}_{i,k},\mathbf{y}_{i,k})\right]\right).
\end{align}

Let $\mathcal{N}_{i,t}$ denote the set of the samples that are not chosen by the $i$-th worker at the $t$-th iteration, $\mathbb{E}[\|\mathbf{o}\|^2]$ can be derived as follows
\begin{align}\label{eq:EoSGD}
\mathbb{E}[\| \mathbf{o}\|^2]=&\mathbb{E}\Bigg[\bigg\|\frac{\sum_{i=1}^U\sum_{k=1}^{K_i}\nabla f(\mathbf{w}_{t-1};\mathbf{x}_{i,k},\mathbf{y}_{i,k})}{K}
+\left(\alpha\sum_{i=1}^U K_b\bm{\beta}_{i,t}\odot\bm{b}_{t}\right)^{\odot-1}\odot\mathbf{z}_{t}\nonumber\\ \nonumber
&-\left(\sum_{i=1}^U K_b\bm{\beta}_{i,t}\right)^{\odot-1}\odot \sum_{i=1}^U\left(\bm{\beta}_{i,t}\odot\mathbb{E}_{\mathcal{D}_i}\left[\sum_{k=1}^{K_b}\nabla f(\mathbf{w}_{t-1},\mathbf{x}_{i,k},\mathbf{y}_{i,k})\right]\right)  \bigg\|^2\Bigg] \\
=&\mathbb{E}\Bigg[\bigg\|\sum_{i=1}^U\left(\frac{\mathbf{1}}{K}-\bm{\beta}_{i,t}\odot\left(\sum_{i=1}^U K_b\bm{\beta}_{i,t}\right)^{\odot-1}\right)\odot \mathbb{E}_{\overline{\mathcal{N}}_{i,t}}\left[\sum_{k\in \overline{\mathcal{N}}_{i,t}}\nabla f(\mathbf{w}_{t-1};\mathbf{x}_{i,k},\mathbf{y}_{i,k})\right]\nonumber\\
&+\frac{\sum_{i=1}^U\mathbb{E}[\sum_{k\in \mathcal{N}_{i,t}}\nabla f(\mathbf{w}_{t-1};\mathbf{x}_{i,k},\mathbf{y}_{i,k})]}{K}
+\left(\alpha\sum_{i=1}^U K_b\bm{\beta}_{i,t}\odot\bm{b}_{t}\right)^{\odot-1}\odot\mathbf{z}_{t} \bigg\|^2\Bigg],\nonumber\\
\leq &\left(\sum_{i=1}^UK_b\right)\sum_{i=1}^U\bigg\|\frac{\mathbf{1}}{K}-\bm{\beta}_{i,t}\odot\left(\sum_{i=1}^U K_b\bm{\beta}_{i,t}\right)^{\odot-1}\bigg\|^2 \mathbb{E}_{\overline{\mathcal{N}}_{i,t}}\left[\sum_{k\in \overline{\mathcal{N}}_{i,t}}\|\nabla f(\mathbf{w}_{t-1};\mathbf{x}_{i,k},\mathbf{y}_{i,k})\|^2\right]\nonumber\\
&+\frac{\|\sum_{i=1}^U\mathbb{E}[\sum_{k\in \mathcal{N}_{i,t}}\nabla f(\mathbf{w}_{t-1};\mathbf{x}_{i,k},\mathbf{y}_{i,k})]\|^2}{K^2}
+\left\|\left(\sum_{i=1}^U K_b\bm{\beta}_{i,t}\odot\bm{b}_{t}\right)^{\odot-1}\right\|^2\sigma^2L^2.
\end{align}

Applying \textbf{Assumption 3}, we get
\begin{align}\label{eq:EoSGDass4}
\mathbb{E}[\| \mathbf{o}\|^2]
\leq &\left(\sum_{i=1}^UK_b\right)\sum_{d=1}^{D}\left(\frac{\left(\sum_{i=1}^UK_b\right)-2K}{K^2}+\frac{1}{\sum_{i=1}^U K_b\beta^d_{i,t}}\right)(\rho_1+\rho_2\|\nabla F(\mathbf{w}_{t-1})\|^2)+\nonumber\\
&\frac{(\sum_{i=1}^{U}(K_i-K_b))^2}{K^2}  (\rho_1+\rho_2\|\nabla F(\mathbf{w}_{t-1})\|^2)
+\left\|\left(\sum_{i=1}^U K_b\bm{\beta}_{i,t}\odot\bm{b}_{t}\right)^{\odot-1}\right\|^2\sigma^2L^2.
\end{align}

Substituting \eqref{eq:EoSGDass4} into \eqref{eq:Egt}, we have
\begin{align}\label{eq:Egt1SGD}
\mathbb{E}[F(\mathbf{w}_{t})]
\leq&\frac{1}{2L}\Bigg(\left\|\left(\sum_{i=1}^U K_b\bm{\beta}_{i,t}\odot\bm{b}_{t}\right)^{\odot-1}\right\|^2\sigma^2L^2\nonumber\\
&+\Bigg(\sum_{d=1}^{D}\Bigg(\frac{\left(\sum_{i=1}^UK_b\right)^2-2K\left(\sum_{i=1}^UK_b\right)}{K^2}+\frac{\left(\sum_{i=1}^UK_b\right)}{\sum_{i=1}^U K_b\beta^d_{i,t}}\Bigg)\nonumber\\&+\frac{(\sum_{i=1}^{U}(K_i-K_b))^2}{K^2}\Bigg)\Bigg)(\rho_1+\rho_2\|\nabla F(\mathbf{w}_{t-1})\|^2)\Bigg) \nonumber\\
& +\mathbb{E}[F(\mathbf{w}_{t-1})]-\frac{1}{2L}\|\nabla F(\mathbf{w}_{t-1})\|^2.
\end{align}

Subtracting $\mathbb{E}[F(\mathbf{w}^*)]$ from both sides of \eqref{eq:Egt1SGD}, and applying \eqref{eq:minimizationlr} and \eqref{eq:minimizationSmooth}, we get
\begin{align}\label{eq:EgtsmoothlySGD}
\mathbb{E}[F(\mathbf{w}_{t})-F(\mathbf{w}^*)]
\leq B^{SGD}_{t}+A^{SGD}_{t}\mathbb{E}[F(\mathbf{w}_{t-1})-F(\mathbf{w}^*)],
\end{align}
where
\begin{align}
A^{SGD}_{t}=&1-\frac{\mu}{L}+\rho_2\Bigg(\sum_{d=1}^{D}\Bigg(\frac{(\sum_{i=1}^UK_b)^2 -2K(\sum_{i=1}^UK_b)}{K^2}+\frac{(\sum_{i=1}^UK_b)}{\sum_{i=1}^U K_b\beta^d_{i,t}}\Bigg)\nonumber\\ &+\frac{(\sum_{i=1}^{U}(K_i-K_b))^2}{K^2}\Bigg),
\end{align}
\begin{align}
B^{SGD}_{t}=&\frac{\rho_1}{2L}\Bigg(\sum_{d=1}^{D}\Bigg(\frac{(\sum_{i=1}^UK_b)^2 -2K(\sum_{i=1}^UK_b)}{K^2}+\frac{(\sum_{i=1}^UK_b)}{\sum_{i=1}^U K_b\beta^d_{i,t}}\Bigg)\nonumber\\&+\frac{(\sum_{i=1}^{U}(K_i-K_b))^2}{K^2}\Bigg)+\left\|\left(\sum_{i=1}^U K_i\bm{\beta}_{i,t}\odot\bm{b}_{t}\right)^{\odot-1}\right\|^2\frac{L\sigma^2}{2}.
\end{align}

Applying \eqref{eq:EgtsmoothlySGD} recursively, we have
\begin{align}\label{eq:SGDConver}
\mathbb{E}[F(\mathbf{w}_{t})-F(\mathbf{w}^*)]
\leq& \sum_{i=1}^{t-1}\prod_{j=1}^i A^{SGD}_{t+1-j}B^{SGD}_{t-i}+B^{SGD}_{t}
+\prod_{j=1}^{t}A^{SGD}_{j}\mathbb{E}[F(\mathbf{w}_{0})-F(\mathbf{w}^*)].
\end{align}
which completes the proof.
\section{Proof of \textbf{Theorem \ref{Theorem_SolutionSpace}}}\label{Appendix D}
To minimize $R_{t}$, it can be seen from \eqref{eq:entryGD}, \eqref{eq:entryNC} and \eqref{eq:entrySGD} that we should maximize the number of the selected workers and the transmit power scaling factor in the $t$-th iteration. Thus, the selected workers should send their parameters at their maximum power. In order to reach the desired parameter aggregation at the PS as in \eqref{eq:g}, each worker needs to use the same transmit power scaling factor $b_{t}$, which is a parameter that needs to be optimized ($b_t$ determines the worker selection). According to \eqref{eq:entryGD}, \eqref{eq:entryNC} and \eqref{eq:entrySGD}, a larger $b_t$ leads to a smaller $R_t$. On the other hand, \eqref{constraint:P3b} indicates that a larger $b_t$ results in less workers is selected, which then results in an increase of $R_t$.


Rewriting \eqref{con:pmax} and replacing $|w_{i,t}|$ with $(|w_{t-1}|+\eta)$, we obtain the maximum acceptable $b_{t}$ of the $i$-th worker as
\begin{align}\label{maxacc}
b_{i,t}^{\max}=\left|\frac{\sqrt{P_i^{\max}}h_{i,t}}{K_i(|w_{t-1}|+\eta)}\right|.
\end{align}

Accordingly, $b_{t}$ should be chosen from $\{b_{i,t}^{\max}\}_{i=1}^U$. Once $b_{t}$ is determined, $\bm{\beta}_{t}$ can be determined by verifying whether the transmit power meets the condition in \eqref{power_limitation}. As a result, we obtain a reduced solution space of the optimization problem \textbf{P2} as
\begin{align}\label{spaceProof}
\mathcal{S}=&\Bigg \{\{(b_t^{(k)},\beta_{i,t}^{(k)})\}_{k=1}^U\bigg|b_t^{(k)}=b_{k,t}^{\max} ,\bm{\beta}_t^{(k)}(b_t^{(k)})=[\beta^{(k)}_{1,t},\dots, \beta^{(k)}_{U,t} ], k = 1,\dots, U\Bigg\},
\end{align}
with
\begin{align}\label{eq:betak}
\beta^{(k)}_{U,t}=H\bigg(P_U^{\max}-\bigg|\frac{K_Ub^{(k)}_{t}(|w_{t-1}|+\eta)}{h_{U,t}}\bigg|\bigg)
\end{align}
where
\begin{equation}
\label{eqH}
H(x)=\left\{
\begin{aligned}
1 & , & x>0, \\
0 & , & x\leq0.
\end{aligned}
\right.
\end{equation}
is the Heaviside step function.


\end{appendices}

\bibliographystyle{IEEEtran}
\bibliography{ref}

\begin{thebibliography}{10}
\providecommand{\url}[1]{#1}
\csname url@samestyle\endcsname
\providecommand{\newblock}{\relax}
\providecommand{\bibinfo}[2]{#2}
\providecommand{\BIBentrySTDinterwordspacing}{\spaceskip=0pt\relax}
\providecommand{\BIBentryALTinterwordstretchfactor}{4}
\providecommand{\BIBentryALTinterwordspacing}{\spaceskip=\fontdimen2\font plus
\BIBentryALTinterwordstretchfactor\fontdimen3\font minus
  \fontdimen4\font\relax}
\providecommand{\BIBforeignlanguage}[2]{{%
\expandafter\ifx\csname l@#1\endcsname\relax
\typeout{** WARNING: IEEEtran.bst: No hyphenation pattern has been}%
\typeout{** loaded for the language `#1'. Using the pattern for}%
\typeout{** the default language instead.}%
\else
\language=\csname l@#1\endcsname
\fi
#2}}
\providecommand{\BIBdecl}{\relax}
\BIBdecl

\bibitem{chiang2016fog}
M.~Chiang and T.~Zhang, ``Fog and iot: An overview of research opportunities,''
  \emph{IEEE Internet of Things Journal}, vol.~3, no.~6, pp. 854--864, 2016.

\bibitem{park2019wireless}
J.~Park, S.~Samarakoon, M.~Bennis, and M.~Debbah, ``Wireless network
  intelligence at the edge,'' \emph{Proceedings of the IEEE}, vol. 107, no.~11,
  pp. 2204--2239, 2019.

\bibitem{li2020federated}
T.~Li, A.~K. Sahu, A.~Talwalkar, and V.~Smith, ``Federated learning:
  Challenges, methods, and future directions,'' \emph{IEEE Signal Processing
  Magazine}, vol.~37, no.~3, pp. 50--60, 2020.

\bibitem{mcmahan2016communication}
H.~B. McMahan, E.~Moore, D.~Ramage, S.~Hampson \emph{et~al.},
  ``Communication-efficient learning of deep networks from decentralized
  data,'' \emph{arXiv preprint arXiv:1602.05629}, 2016.

\bibitem{konevcny2016federated}
J.~Kone{\v{c}}n{\`y}, H.~B. McMahan, D.~Ramage, and P.~Richt{\'a}rik,
  ``Federated optimization: Distributed machine learning for on-device
  intelligence,'' \emph{arXiv preprint arXiv:1610.02527}, 2016.

\bibitem{zhu2020toward}
G.~Zhu, D.~Liu, Y.~Du, C.~You, J.~Zhang, and K.~Huang, ``Toward an intelligent
  edge: wireless communication meets machine learning,'' \emph{IEEE
  Communications Magazine}, vol.~58, no.~1, pp. 19--25, 2020.

\bibitem{chen2020joint}
M.~Chen, Z.~Yang, W.~Saad, C.~Yin, H.~V. Poor, and S.~Cui, ``A joint learning
  and communications framework for federated learning over wireless networks,''
  \emph{IEEE Transactions on Wireless Communications}, 2020.

\bibitem{vu2020cell}
T.~T. Vu, D.~T. Ngo, N.~H. Tran, H.~Q. Ngo, M.~N. Dao, and R.~H. Middleton,
  ``Cell-free massive mimo for wireless federated learning,'' \emph{IEEE
  Transactions on Wireless Communications}, 2020.

\bibitem{nazer2007computation}
B.~Nazer and M.~Gastpar, ``Computation over multiple-access channels,''
  \emph{IEEE Transactions on information theory}, vol.~53, no.~10, pp.
  3498--3516, 2007.

\bibitem{chen2018over}
L.~Chen, N.~Zhao, Y.~Chen, F.~R. Yu, and G.~Wei, ``Over-the-air computation for
  iot networks: Computing multiple functions with antenna arrays,'' \emph{IEEE
  Internet of Things Journal}, vol.~5, no.~6, pp. 5296--5306, 2018.

\bibitem{goldenbaum2013harnessing}
M.~Goldenbaum, H.~Boche, and S.~Sta{\'n}czak, ``Harnessing interference for
  analog function computation in wireless sensor networks,'' \emph{IEEE
  Transactions on Signal Processing}, vol.~61, no.~20, pp. 4893--4906, 2013.

\bibitem{abari2015airshare}
O.~Abari, H.~Rahul, D.~Katabi, and M.~Pant, ``Airshare: Distributed coherent
  transmission made seamless,'' in \emph{2015 IEEE Conference on Computer
  Communications (INFOCOM)}.\hskip 1em plus 0.5em minus 0.4em\relax IEEE, 2015,
  pp. 1742--1750.

\bibitem{amiri2020machine}
M.~M. Amiri and D.~G{\"u}nd{\"u}z, ``Machine learning at the wireless edge:
  Distributed stochastic gradient descent over-the-air,'' \emph{IEEE
  Transactions on Signal Processing}, vol.~68, pp. 2155--2169, 2020.

\bibitem{amiri2020federated}
------, ``Federated learning over wireless fading channels,'' \emph{IEEE
  Transactions on Wireless Communications}, vol.~19, no.~5, pp. 3546--3557,
  2020.

\bibitem{amiri2019collaborative}
M.~M. Amiri, T.~M. Duman, and D.~G{\"u}nd{\"u}z, ``Collaborative machine
  learning at the wireless edge with blind transmitters,'' \emph{arXiv preprint
  arXiv:1907.03909}, 2019.

\bibitem{zhu2019broadband}
G.~Zhu, Y.~Wang, and K.~Huang, ``Broadband analog aggregation for low-latency
  federated edge learning,'' \emph{IEEE Transactions on Wireless
  Communications}, vol.~19, no.~1, pp. 491--506, 2019.

\bibitem{sun2019energy}
Y.~Sun, S.~Zhou, and D.~G{\"u}nd{\"u}z, ``Energy-aware analog aggregation for
  federated learning with redundant data,'' \emph{arXiv preprint
  arXiv:1911.00188}, 2019.

\bibitem{yang2020federated}
K.~Yang, T.~Jiang, Y.~Shi, and Z.~Ding, ``Federated learning via over-the-air
  computation,'' \emph{IEEE Transactions on Wireless Communications}, vol.~19,
  no.~3, pp. 2022--2035, 2020.

\bibitem{ye2018communication}
M.~Ye and E.~Abbe, ``Communication-computation efficient gradient coding,''
  \emph{arXiv preprint arXiv:1802.03475}, 2018.

\bibitem{aji2017sparse}
A.~F. Aji and K.~Heafield, ``Sparse communication for distributed gradient
  descent,'' \emph{arXiv preprint arXiv:1704.05021}, 2017.

\bibitem{liu2019decentralized}
Y.~Liu, K.~Yuan, G.~Wu, Z.~Tian, and Q.~Ling, ``Decentralized dynamic admm with
  quantized and censored communications,'' in \emph{2019 53rd Asilomar
  Conference on Signals, Systems, and Computers}.\hskip 1em plus 0.5em minus
  0.4em\relax IEEE, 2019, pp. 1496--1500.

\bibitem{liu2019communication}
Y.~Liu, W.~Xu, G.~Wu, Z.~Tian, and Q.~Ling, ``Communication-censored admm for
  decentralized consensus optimization,'' \emph{IEEE Transactions on Signal
  Processing}, vol.~67, no.~10, pp. 2565--2579, 2019.

\bibitem{8755802}
P.~{Xu}, Z.~{Tian}, Z.~{Zhang}, and Y.~{Wang}, ``Coke: Communication-censored
  kernel learning via random features,'' in \emph{2019 IEEE Data Science
  Workshop (DSW)}, 2019, pp. 32--36.

\bibitem{chen2018lag}
T.~Chen, G.~Giannakis, T.~Sun, and W.~Yin, ``Lag: Lazily aggregated gradient
  for communication-efficient distributed learning,'' in \emph{Advances in
  Neural Information Processing Systems}, 2018, pp. 5050--5060.

\bibitem{8646657}
P.~{Xu}, Z.~{Tian}, and Y.~{Wang}, ``An energy-efficient distributed average
  consensus scheme via infrequent communication,'' in \emph{2018 IEEE Global
  Conference on Signal and Information Processing (GlobalSIP)}, 2018, pp.
  648--652.

\bibitem{xu2020coke}
P.~Xu, Y.~Wang, X.~Chen, and T.~Zhi, ``Coke: Communication-censored kernel
  learning for decentralized non-parametric learning,'' \emph{arXiv preprint
  arXiv:2001.10133}, 2020.

\bibitem{zeng2019energy}
Q.~Zeng, Y.~Du, K.~K. Leung, and K.~Huang, ``Energy-efficient radio resource
  allocation for federated edge learning,'' \emph{arXiv preprint
  arXiv:1907.06040}, 2019.

\bibitem{wang2018cooperative}
J.~Wang and G.~Joshi, ``Cooperative sgd: A unified framework for the design and
  analysis of communication-efficient sgd algorithms,'' \emph{arXiv preprint
  arXiv:1808.07576}, 2018.

\bibitem{goldenbaum2013robust}
M.~Goldenbaum and S.~Stanczak, ``Robust analog function computation via
  wireless multiple-access channels,'' \emph{IEEE Transactions on
  Communications}, vol.~61, no.~9, pp. 3863--3877, 2013.

\bibitem{shamir2014communication}
O.~Shamir, N.~Srebro, and T.~Zhang, ``Communication-efficient distributed
  optimization using an approximate newton-type method,'' in
  \emph{International conference on machine learning}, 2014, pp. 1000--1008.

\bibitem{magnusson2020maintaining}
S.~Magn{\'u}sson, H.~Shokri-Ghadikolaei, and N.~Li, ``On maintaining linear
  convergence of distributed learning and optimization under limited
  communication,'' \emph{IEEE Transactions on Signal Processing}, 2020.

\bibitem{Bertsekas1996Neuro}
D.~P. Bertsekas, J.~N. Tsitsiklis, and J.~Tsitsiklis, \emph{Neuro-Dynamic
  Programming}.\hskip 1em plus 0.5em minus 0.4em\relax Athena Scientific, 1996.

\bibitem{friedlander2012hybrid}
M.~P. Friedlander and M.~Schmidt, ``Hybrid deterministic-stochastic methods for
  data fitting,'' \emph{SIAM Journal on Scientific Computing}, vol.~34, no.~3,
  pp. A1380--A1405, 2012.

\bibitem{alistarh2018convergence}
D.~Alistarh, T.~Hoefler, M.~Johansson, N.~Konstantinov, S.~Khirirat, and
  C.~Renggli, ``The convergence of sparsified gradient methods,'' in
  \emph{Advances in Neural Information Processing Systems}, 2018, pp.
  5973--5983.

\bibitem{yuan2016convergence}
K.~Yuan, Q.~Ling, and W.~Yin, ``On the convergence of decentralized gradient
  descent,'' \emph{SIAM Journal on Optimization}, vol.~26, no.~3, pp.
  1835--1854, 2016.

\bibitem{bottou2018optimization}
L.~Bottou, F.~E. Curtis, and J.~Nocedal, ``Optimization methods for large-scale
  machine learning,'' \emph{Siam Review}, vol.~60, no.~2, pp. 223--311, 2018.

\bibitem{stich2018sparsified}
S.~U. Stich, J.-B. Cordonnier, and M.~Jaggi, ``Sparsified sgd with memory,'' in
  \emph{Advances in Neural Information Processing Systems}, 2018, pp.
  4447--4458.

\bibitem{tang2019doublesqueeze}
H.~Tang, C.~Yu, X.~Lian, T.~Zhang, and J.~Liu, ``Doublesqueeze: Parallel
  stochastic gradient descent with double-pass error-compensated compression,''
  in \emph{International Conference on Machine Learning}.\hskip 1em plus 0.5em
  minus 0.4em\relax PMLR, 2019, pp. 6155--6165.

\bibitem{alistarh2017qsgd}
D.~Alistarh, D.~Grubic, J.~Li, R.~Tomioka, and M.~Vojnovic, ``Qsgd:
  Communication-efficient sgd via gradient quantization and encoding,'' in
  \emph{Advances in Neural Information Processing Systems}, 2017, pp.
  1709--1720.

\end{thebibliography}

\end{document}